\documentclass[pmlr, twocolumn, 10pt]{jmlr} 
\usepackage{makecell}
\usepackage{threeparttable}
\usepackage{booktabs}
\usepackage{multirow}
\usepackage{longtable}
\usepackage{centernot}
\usepackage{caption}

\newcommand{\decondtntoolkit}{\textsc{DeconDTN-Toolkit}\xspace}

\usepackage{booktabs}
\usepackage{siunitx}

\newcommand{\equal}[1]{{\hypersetup{linkcolor=black}\thanks{#1}}}

\DeclareMathOperator*{\argmin}{arg\,min}

\usepackage{listings}
\usepackage{xcolor}

\jmlrvolume{333}
\jmlryear{2026}
\jmlrsubmitted{LEAVE UNSET}
\jmlrpublished{LEAVE UNSET}
\jmlrworkshop{Conference on Health, Inference, and Learning (CHIL)
2026} 

\title[DeconDTN-Toolkit]{DeconDTN-Toolkit: A Library for Evaluation
and Enhancement of Robustness to Provenance Shift}

\author{%
  \Name{Yongsen Tan\textsuperscript{$\dagger$}}\equal{These authors
  contributed equally}
  \Email{tanys@uw.edu}
  \AND
  \Name{Zhecheng Sheng\textsuperscript{$\ddagger$}}\footnotemark[1]
  \Email{sheng136@umn.edu}
  \AND
  \Name{Xiruo Ding\textsuperscript{$\dagger$}}\footnotemark[1]
  \Email{xiruod@uw.edu}
  \AND
  \Name{Serguei V. S. Pakhomov\textsuperscript{$\ddagger$}}
  \Email{pakh0002@umn.edu}
  \AND
  \Name{Trevor Cohen\textsuperscript{$\dagger$}} \Email{cohenta@uw.edu}\\
  \addr
  \textsuperscript{$\dagger$}University of Washington,
  \textsuperscript{$\ddagger$}University of Minnesota
}

\begin{document}
\maketitle

\begin{abstract}
  Despite the burgeoning body of work on distribution shifts,
  provenance shift—where the relationship between data source and
  label changes at deployment—remains poorly understood and under-addressed.
  In this paper, we establish a formal connection between
  provenance shift, counterfactual invariance, and invariant learning
  to derive a learning objective for robustness.
  We then introduce \decondtntoolkit, a specialized evaluation and
  remediation suite designed to simulate provenance shifts of varying
  degrees while maintaining the training protocol and the
  infrastructure of existing benchmarks.
  We reveal the vulnerability of Empirical Risk Minimization under
  provenance shift, introduce a robust out-of-distribution
  performance indicator, and conduct a comprehensive evaluation on
  existing algorithms.
  Our work provides both the theoretical grounding and the practical
  tools necessary to characterize the problem of confounding by
  provenance, and implementations of methods to mitigate it.
\end{abstract}

\paragraph*{Data and Code Availability}
This paper uses five publicly available datasets:
\textbf{SHAC} \citep{lybarger2021annotating} is
available upon request,
\textbf{MIMIC-III} \citep{johnson2016mimic} is
available through
PhysioNet\footnote{\url{https://physionet.org/content/mimiciii/1.4}},
\textbf{HateSpeech} \citep{vidgen2021learning,
de2018hate} is available on the
investigators' GitHub
repository\footnote{\url{https://github.com/Vicomtech/hate-speech-dataset}}\footnote{\url{https://github.com/bvidgen/Dynamically-Generated-Hate-Speech-Dataset}},
\textbf{\textcolor{gray}{MultiNLI}} \citep{williams2018broad} via
Huggingface\footnote{\url{https://huggingface.co/datasets/nyu-mll/multi_nli}},
and \textbf{\textcolor{gray}{Civilcomments}}
\citep{borkan2019nuanced} is available through
Kaggle\footnote{\url{https://www.kaggle.com/competitions/jigsaw-unintended-bias-in-toxicity-classification}}.
We include a detailed data access and usage instructions in
\decondtntoolkit
(\url{https://github.com/LinguisticAnomalies/DeconDTN-toolkit}).

\paragraph*{Author Contributions}
Y.T. developed the theoretical framework, led the toolkit
development, performed the experiments, analyzed the data, and wrote
the manuscript.
Z.S. and X.D. developed the initial toolkit, aided in toolkit
refactorization, conceived the evaluation framework, and contributed to
the interpretation of the results.
T.C. and S.P. conceived the study and were in charge of overall
direction and planning.
All authors provided critical feedback and helped shape the research,
analysis, and manuscript.

\paragraph*{Institutional Review Board (IRB)}
This study was approved by the University of Washington IRB Committee
(protocol ID STUDY00015108).

\section{Introduction}
The success of modern machine learning systems relies critically on the quality
and quantity of their training data.
A common practice to scale the data quantity and reduce generalization error in
natural language processing (NLP) is to mix multiple data sources for
a unified task \citep{laparra2020rethinking}.
For example, the pre-training corpus for masked language modeling and
next sequence prediction tasks of BERT were BooksCorpus and English
Wikipedia \citep{devlin2019bert}, and more recent efforts have sought
to develop and validate clinical NLP algorithms collaboratively
within an established data-sharing network \citep{wang2025development}.
Counter-intuitively, increasing the number of data sources can harm systems'
performance by introducing spurious correlations to the training distribution in
real-world settings \citep{guo2021crossing, compton2023more, shen2024data}.
A mechanism through which machine learning models might internalize
such spurious correlations was revealed by the “Name That Dataset”
experiment, in which a neural network was trained to identify the
source of an image sampled from a variety of datasets
\citep{torralbaUnbiasedLookDataset2011}.
Notably, the dataset an image was drawn from could be identified with
high accuracy in 2011, using deep learning methods
\citep{krizhevskyImageNetClassificationDeep2012, liuDecadesBattleDataset2024}.
Modern neural networks, therefore, can harness source-specific
features for label prediction
\citep{geirhosShortcutLearningDeep2020}, resulting in flawed systems
that encode spurious correlations between data sources and labels.
For example, in the context of a multi-site set of clinical notes, a
system may learn to associate a characteristic acronym with the
prevalence of an outcome of interest at one location, which would
lead to inaccurate predictions at the point of deployment if the
differences in prevalence across sites do not match those in the
training set
\citep{howell2020controlling,kohWILDSBenchmarkIntheWild2021,yangChangeHardCloser2023}.
This algorithmic bias has been referred to as \textit{confounding by
provenance} in previous work \citep{ding2023enhancing}.

There is a burgeoning body of work on addressing robustness to
test-time distribution shifts, including domain generalization
\citep{zhouDomainGeneralizationSurvey2023}, label shift
\citep{liptonDetectingCorrectingLabel2018}, and subpopulation shift
\citep{kohWILDSBenchmarkIntheWild2021,yangChangeHardCloser2023}.
However, this prior work does not focus on provenance shift, and the
utility of methods developed to address these distribution shifts in
the context of confounding by provenance has yet to be established.
In addition, although recent theoretical and algorithmic frameworks
have been proposed to account for observed confounders in machine
learning \citep{landeiro2016robust,landeiro2018robust,
  veitchCounterfactualInvarianceSpurious2021a, schrouff2024mind,
zhang2024causal}, there is an unmet need for a unified toolkit to
facilitate the systematic evaluation of robustness to provenance
shift, and evaluate approaches with the potential to address it.

In this work, we provide a systematic approach to the study of
robustness to provenance shift in the context of multiple data sources.
We first introduce the problem of provenance shift and establish a
formal connection between it and the related domains of
counterfactual invariance and invariant learning by deriving a
robustness learning objective under provenance shift
(\sectionref{section:problem}).
To facilitate empirical study of provenance shift, we introduce the
\decondtntoolkit (Deconfounding Deep Transformer Networks Toolkit),
which can synthetically introduce varying degrees of provenance
shift between train and test time and provides a standardized
interface to a range of established and recently-developed approaches
(\sectionref{section:decondtntoolkit}).
Using \decondtntoolkit, we investigate the learning dynamics of
Empirical Risk Minimization (ERM), identify a predictor of
out-of-distribution performance, and test for robustness a set of
invariant learning algorithms under increasing degrees of provenance
shift (\sectionref{section:setup,section:results}).
Our contributions are summarized as follows:
\begin{enumerate}
  \item \decondtntoolkit, a specialized and systematic suite to
    simulate provenance shifts of varying degrees, evaluate the
    robustness of text classification models, and remediate
    vulnerability to provenance shift using invariant learning algorithms.
  \item We formalize provenance shift, derive a risk-invariant
    objective under specific causal graphs, and, through a
    comprehensive empirical study with \decondtntoolkit, demonstrate
    that standard ERM and oracle selection fail under such shifts.
  \item A comprehensive benchmarking on invariant learning algorithms
    under provenance shift with aligned evaluation protocol to previous study.
    We show that removing the correlation between provenance $Z$ and
    outcome $Y$ in the training distribution remains strong baseline
    performance under provenance shift.
\end{enumerate}

\section{Problem Setting}
\label{section:problem}
To formalize the problem of provenance shift, we consider a
supervised learning task where the goal is to learn a predictor
$f: \mathcal{X} \rightarrow \mathcal{Y}$ using a labeled dataset
$\mathcal{D} = \{(x_i, y_i, z_i) \}_{i=1}^n$, where $Z$ denotes
observed confounders.
In the problem of provenance shift, $Z$ denotes the provenance of a
subset of data within  $\mathcal{D}$.
Given a loss function $\mathcal{L}: \mathcal{Y} \times \mathcal{Y}
\rightarrow [0, \infty)$
defining the learning objective, supervised learning seeks a predictor $f$ that
minimizes the risk $\mathbb{E}_{(x,y)\sim P}[\mathcal{L}(f(x),y)]$.
We assume that samples are independent and identically distributed (i.i.d.)
according to the causal graph $\mathcal{G}: X \leftarrow Z \rightarrow Y$.

This work investigates the problem of learning a predictor $f$ that is
counterfactually invariant to the confounder $Z$ (e.g. $f(X(z)) =
f(X(z'))$), thereby achieving robustness to provenance shift at test time.
In the following sections, we first define, decompose, and
characterize provenance shift (\sectionref{section:provenance-shift}).
Then, we connect counterfactual invariance to provenance shift,
derive a risk-invariant learning objective under specific causal
graphs, and propose practical approaches to achieve the objective
(\sectionref{section:robustness-provenance-shift}).

\subsection{Provenance Shift}
\label{section:provenance-shift}
Provenance shift refers to a test-time change in the correlation between the
provenance $Z$ and the outcome $Y$\footnote{Provenance shift is a
form of \textit{confounding shift} as defined by \citet{landeiro2018robust}.}.
We assume that the underlying causal graph $\mathcal{G}$ remains
invariant across the training distribution $P^{tr}$ and the test distribution
$P^{te}$.

\begin{definition}[Provenance Shift]
  \label{def:provenance-shift}
  Let $Z$ be an observed confounder (i.e., common cause) of the
  outcome $Y$ and the observed variables
  $X$ in the causal graph $\mathcal{G}: X \leftarrow Z \rightarrow Y$.
  Provenance shift occurs when $P^{tr}(Y \mid Z) \neq P^{te}(Y \mid Z)$, while
  $\mathcal{G}$ remains unchanged.
\end{definition}

\begin{example}[Substance]
  \label{ex:substance}
  Consider detecting substance abuse mentions $Y$ in clinical notes
  $X$ from two sites $Z \in \{z_1, z_2\}$.
  In the training environment, site $z_1$ utilizes purposive
  sampling, resulting in a high prevalence of mentions $P^{tr}(Y \mid Z=z_1)$.
  During deployment, site $z_1$ switches to standard clinical flow,
  where the prevalence $P^{te}(Y \mid Z=z_1)$ is significantly lower
  and no longer artificially elevated.
  Even if site $z_2$ remains stable, the system may overestimate
  substance abuse for notes from $z_1$.
\end{example}
\begin{example}[Goals-of-Care]
  \label{ex:discussion}
  Consider detecting goals-of-care discussions, a prerequisite to
  goal-concordant end-of-life care, from clinical notes.
  Let $Z \in \{z_1, z_2\}$ denote the note provenance, where $z_1$
  corresponds to palliative care specialist notes and $z_2$ to notes
  from a random sample of emergency admissions.
  In the training environment, a convenience sample may overrepresent
  palliative care notes, yielding a higher prevalence of
  goals-of-care discussions $P^{tr}(Y \mid Z=z_1)$.
  During deployment, however, site $z_1$ may shift to random
  universal screening, causing $P^{te}(Y \mid Z=z_1)$ to drop
  significantly as the "purposive" focus on high-risk patients is removed.
  Conversely, site $z_2$ might open a "New Wing" dedicated to
  intensive care planning, causing the prevalence $P^{te}(Y \mid
  Z=z_2)$ to rise relative to the original general admission baseline
  $P^{tr}(Y \mid Z=z_2)$.
  A system deployted across $Z$ can overestimate discussions for site
  $z_1$ and underestimate them for site $z_2$.
\end{example}
From a dataset bias perspective — wherein neural networks can learn
generalizable and transferable features of provenance
\citep{liu2024decade}, we decompose the discriminative model that
follows the causal graph $\mathcal{G}$ into two components: the
inference mechanism (the capacity to map spurious features back to
their latent causes) and the provenance mechanism (the observed
outcome prevalence conditioned on provenance):
\begin{lemma}(Prediction Decomposition)
  \label{lemma:decomposition}
  Let $X^\perp_Y$ be $Y$-invariant components of the input features, such that
  $X^\perp_Y(y) = X^\perp_Y(y')$.
  A discriminative model under the causal graph $\mathcal{G}: X
  \leftarrow Z \rightarrow Y$
  can be decomposed into two components: the inference and the
  provenance mechanism.
  \begin{equation}
    P(Y \mid X^\perp_Y)
    = \int
    \underbrace{P(Y \mid Z)}_{\substack{\text{provenance}}}
    \,
    \underbrace{P(Z \mid X^\perp_Y)}_{\substack{\text{inference}}}
    \, dZ
  \end{equation}
\end{lemma}
Detailed proofs are provided in \ref{apd:proof:lemma2}.
The instability of $f$ under provenance shift stems from features in
$X$ that serve as proxies for $Z$ (i.e., inference mechanism).
Intuitively, $Z$ serves as an intermediary between $X^\perp_Y$ and
$Y$, and the prediction shifts if the provenance-specific label
distribution changes (\definitionref{def:provenance-shift}) while the
inference mechanism is generalizable for predictor $f$ at test-time.
Because $Z$ is predictive of $Y$ in the training distribution, the
model may rely on these proxies $X^\perp_Y$ despite the absence of a
direct causal link between the proxy features $X^\perp_Y$ and the outcome $Y$.

\begin{table*}[t]
  \floatconts {tab:concept}
  {\caption{Distribution Shifts in Supervised Learning. $L^z$ and
      $U^z$ denote the labeled and unlabeled distribution from domain
  $z$, respectively}}
  {
    \begin{center}
      {
        \resizebox{\textwidth}{!}{%
          \begin{tabular}{lllll}
            \toprule
            \textbf{Concept} & \textbf{Train Inputs} & \textbf{Test
            Inputs} & \textbf{Test-Time Invariance} &
            \textbf{Test-Time Change} \\
            \midrule
            Label Shift & $L^1$ & $U^1$ & $P^{tr}(Y \mid X) =
            P^{te}(Y \mid X)$ & $P^{tr}(Y) \neq P^{te}(Y)$ \\
            Domain Generalization & $L^1,\ldots, L^{z^{tr}}$ &
            $U^{z^{tr+1}}$ & $P^{tr}(Y \mid X) = P^{te}(Y \mid X)$ &
            $P^{tr}(X) \neq P^{te}(X)$ \\
            Spurious Correlation & $L^1,\ldots, L^{z^{tr}}$ &
            $U^1,\ldots, U^{z^{tr}}$ & $P^{tr}(Y \mid X^\perp_Z) =
            P^{te}(Y \mid X^\perp_Z)$ & $P^{tr}(Y \mid Z) \neq
            P^{te}(Y \mid Z)$ \\
            Subpopulation Shift  & $L^1,\ldots, L^{z^{tr}}$ &
            $U^1,\ldots, U^{z^{tr+n}}$ & $P^{tr}(X,Y \mid
            Z)=P^{te}(X,Y \mid Z)$ & $P^{tr}(Z) \neq P^{te}(Z)$ \\
            \midrule
            \textbf{Provenance Shift}  & $L^1,\ldots, L^{z^{tr}}$ &
            $U^1,\ldots, U^{z^{tr}}$ & $\mathcal{G}^{tr} =
            \mathcal{G}^{te}, X \leftarrow Z \rightarrow Y$ &
            $P^{tr}(Y \mid Z) \neq P^{te}(Y \mid Z)$\\
            \bottomrule
          \end{tabular}
        }
      }
    \end{center}
  }
\end{table*}
We compare provenance shift with a set of related concepts in
\tableref{tab:concept}.
Notably, provenance shift does not assume a distribution shift in
$P(Y)$ or $P(Z)$, nor does it require the model $f$ to generalize to
unseen values of $Z$.
Furthermore, performance degradation under provenance shift does not
strictly require attribute imbalance, class imbalance, or attribute
generalization;
rather, it manifests as a specific form of spurious correlation
resulting from the confounding structure $\mathcal{G}$
\citep{dingBackdoorAdjustmentConfounding2024}.

\subsection{Robustness to Provenance Shift}
\label{section:robustness-provenance-shift}
The learning objective of counterfactual invariance is to learn a
robust predictor which is counterfactual invariant to the confounder $Z$:

\begin{definition}[Counterfactual Invariance]
  (Restated Definition 1 from
  \citet{veitchCounterfactualInvarianceSpurious2021a})
  Denote $X(z)$ as the counterfactual $X$ would have observed had $Z$
  been set to $z$
  via intervention, leaving all other conditions fixed.
  A predictor $f$ is counterfactual invariant to $Z$ if $f(X(z))=f(X(z'))$
  almost everywhere, for all $z, z' \in \mathcal{Z}$.
\end{definition}
To derive a learning objective that ensures robustness under
provenance shift, we establish a formal connection between
counterfactual invariance and provenance shift:
\begin{proposition}[Provenance Robustness]
  \label{proposition:robustness}
  If a predictor $f: \mathcal{X} \rightarrow \mathcal{Y}$ satisfies
  counterfactual invariance such that $f(X(z)) = f(X(z'))$ for all
  $z, z' \in \mathcal{Z}$, then the predictor is robust to provenance shift.
  Specifically, the risk $\mathbb{E}[\mathcal{L}(f(X), Y)]$ remains
  constant under any intervention on the provenance-specific class
  distribution $P(Y \mid Z)$, provided $P(Y)$ remains invariant.
  This robustness holds under both:
  \begin{enumerate}
    \item Anti-causal settings $Y \rightarrow X$;
    \item Causal settings $X \rightarrow Y$: Provided that $Y \perp X
      \mid \{X^\perp_Z, Z\}$
      and the label satisfies counterfactual consistency, $Y(z) =
      Y(z')$ for all $z, z' \in \mathcal{Z}$.
  \end{enumerate}
\end{proposition}
Detailed proofs are provided in Appendix \ref{apd:proof:proposition4}.
While counterfactuals are often unobservable, we can achieve the
necessary conditions for provenance robustness by enforcing the
predictor $f$ to satisfy the necessary condition of counterfactual
invariance under the causal graph $\mathcal{G}$ in Proposition
\ref{proposition:robustness}
\citep{veitchCounterfactualInvarianceSpurious2021a}:
\begin{equation}
  f(X) \perp Z \mid Y
\end{equation}
In practice, this learning objective can be achieved by invariant
learning, which aims to achieve an invariant prediction $f(X) \perp Z
\mid Y$ using a labeled dataset $\mathcal{D} = \{(x_i, y_i, z_i)
\}_{i=1}^{n}$ partitioned by attributes $z \in \mathcal{Z}$ (e.g.,
domains in domain generalization or subgroups in subpopulation shift).
To this end, we derive the learning objective of provenance
robustness and the practical approach to achieve it.
If $Z$ is entirely latent and has no proxy, identifying a shift in
$P(Y \mid Z)$ is mathematically intractable. For identifiability, we
assume that the confounder $Z$ is observed during training and
validation and remains unobserved at test time in the current work.

\section{\decondtntoolkit: A Toolkit to Quantify and Improve
Robustness to Provenance Shift}
\label{section:decondtntoolkit}
\decondtntoolkit is an evaluation suite designed to evaluate model
robustness to provenance shift.
The \textit{evaluation} component of the toolkit generates test
splits with increasing degrees of provenance shift systematically,
formalizing the evaluation approach for provenance shift developed by
\citet{landeiro2016robust} following prior work by \citet{ding2023enhancing}.
The \textit{mitigation} component is built as an extension of
\textsc{DomainBed}, the standard domain generalization toolkit for
computer vision \citep{gulrajaniSearchLostDomain2020}.
\decondtntoolkit also inherits the standardized training protocol in
\textsc{DomainBed} and \textsc{SubpopBench} for finding alignment
\citep{gulrajaniSearchLostDomain2020,yangChangeHardCloser2023}.

The initial release of \decondtntoolkit includes ingestion pipelines
for five datasets (\tableref{tab:dataset}) and implementations of
nineteen algorithms, within an extensible infrastructure adapted from
\textsc{DomainBed} to support custom datasets and algorithms.

The user interface of \decondtntoolkit is inspired by the \textsc{TRL}
library\footnote{\url{https://huggingface.co/docs/trl/en/index}},
utilizing a centralized \texttt{Trainer} class to manage algorithms,
datasets, and training configurations.
Following \textsc{DomainBed}, we adopt a provenance-balanced loading strategy
during training; for any minibatch $\mathcal{B}$, the distribution
$P_{\sim \mathcal{B}}(Z)$ is enforced to be uniform.

\subsection{Simulation/Evaluation}
To facilitate evaluation of robustness to confounding by provenance,
\decondtntoolkit can introduce spurious correlations of arbitrary strength
between the label $Y$ and provenance $Z$ by specifying the
$\{P^{tr}(Y, Z), P^{te}(Y, Z)\}$.
If the split sizes $\{|\mathcal{D}^{tr}|, |\mathcal{D}^{te}|\}$
are not explicitly defined, the toolkit employs a greedy subsampling strategy
to maximize the feasible sample size. Finally, we ensure that the sampled corpus
remains i.i.d., with samples from the same subject strictly partitioned into the
same subset to prevent data leakage.

To quantify the degree of shift,  we employ a parameter $\alpha \in
\mathbb{R}^{|Y|\times |Z|}$ developed by \cite{ding2023enhancing} to
measure the correlation between $Y$ and $Z$:
\begin{equation}
  \label{eq:alpha}
  \log\alpha_{ij} = \log\frac{P(Y=i\mid Z=j)}{P(Y=i \mid Z\neq j)}
\end{equation}
$\alpha_{ij}$ describes the ratio of the conditional distribution of outcome $i$
between provenance $j$ versus others.
In binary $Y$ and $Z$ settings, which are the focus of the toolkit
currently, we define
$\log\alpha := \log\frac{P(Y=1\mid Z=1)}{P(Y=1 \mid Z=0)}\in \mathbb{R}$.
Intuitively, $\log\alpha = 0$ represents a jointly balanced distribution,
$\log\alpha > 0$ represents cases in which the prevalence of the
primary outcome in provenance $Z=1$ is greater than another.
The sampling-based simulation framework in \decondtntoolkit
facilitates the generation of training and testing splits with custom
$\alpha$ values. The difference between $\log\alpha$ at training and
test time then provides a measure of the magnitude of the provenance
shift that has been introduced.

\subsection{Algorithms/Mitigation}
The initial release of \decondtntoolkit includes nineteen algorithms as follows:
(1) \underline{Baseline:}
Empirical Risk Minimization (\textbf{ERM});
(2) \underline{Sampling:}
\textbf{UpSampling},
\textbf{DownSampling} \citep{japkowicz2000class};
(3) \underline{Marginal distribution adjustment:}
Backdoor Adjustment (\textbf{BackDoor},
\citet{landeiro2016robust}),
Marginal Transfer Learning (\textbf{MTL}, \citet{blanchard2021domain});
(4) \underline{Data augmentation:}
Mixup (\textbf{Mixup},
\citet{zhangMixupEmpiricalRisk2018a}),
Learning Invariant Predictors with Selective Augmentation (\textbf{LISA},
\citep{yaoImprovingOutofDistributionRobustness2022}),
(5) \underline{Distribution matching:}
Deep Correlation Alignment (\textbf{CORAL},
\citet{sunDeepCORALCorrelation2016}),
Maximum Mean Discrepancy (\textbf{MMD},
\citet{liDomainGeneralizationAdversarial2018}),
Optimal Representations for Covariate Shift (\textbf{CAD},
\citet{ruanOptimalRepresentationsCovariate2021});
(6) \underline{Gradient matching:}
Gradient Matching for Domain Generalization (\textbf{Fish},
\citet{shiGradientMatchingDomain2021});
(7) \underline{Adversarial training:}
Domain Adversarial Neural Network (\textbf{DANN},
\citet{ganinDomainAdversarialTrainingNeural2016}),
Conditional Domain Adversarial Neural Network (\textbf{CDANN},
\citet{liDeepDomainGeneralization2018});
(8) \underline{Invariant feature learning:}
Invariant Risk Minimization (\textbf{IRM},
\citet{arjovskyInvariantRiskMinimization2020});
(9) \underline{Group robust learning:}
Group Distributionally Robust Optimization (\textbf{GroupDRO},
\citet{sagawa*DistributionallyRobustNeural2019});
(10) \underline{Two-stage training:}
Just Train Twice (\textbf{JTT},
\citet{liuJustTrainTwice2021}),
Deep Feature Reweighting (\textbf{DFR},
\citet{kirichenkoLastLayerReTraining2023}),
Learning from Failure (\textbf{LfF},
\citet{nam2020learning}),
Dual Filter (\textbf{DualFilter},
\citet{shengMitigatingConfoundingSpeechBased2025}).
We included a detailed related work in \appendixref{apd:relatedwork}.
Algorithm descriptions and their hyperparameters can be found
in \appendixref{apd:algorithms} and
\tableref{apd:tab:hyperparameters}, respectively.

\begin{table*}[htbp]
  \floatconts{tab:dataset}
  {\caption{Corpora for use in \textsc{DeconDTN-Toolkit}.
      \textbf{Bold} indicate \textbf{Source Datasets} and 
      \textbf{\textcolor{gray}{Light}} indicate
      \textbf{\textcolor{gray}{Attribute Datasets}}.}}
  {
    \begin{center}
      {
        \begin{tabular}{lllll}
          \toprule
          \textbf{Dataset} & \textbf{Prediction ($Y$)} &
          \textbf{Attribute ($Z$)} & \textbf{Sample size} &
          \textbf{Adapted from}\\
          \midrule
          \textbf{SHAC} & Drug Abuse & Data
          Source & 4,405 & \citet{lybarger2021annotating} \\
          \textbf{MIMIC-Location} & Mortality &
          Data Source & 16,282  & \citet{johnson2016mimic} \\
          \textbf{HateSpeech} & Hate Speech &
          Data Source & 51,847 &
          \makecell[l]{\citet{vidgen2021learning}, \\ \citet{de2018hate}}\\
          \midrule
          \textbf{\textcolor{gray}{\textbf{\textcolor{gray}{MIMIC-SubpopBench}}}} & Mortality &
          Sex & 25,880 & \citet{yangChangeHardCloser2023} \\
          \textbf{\textcolor{gray}{\textbf{\textcolor{gray}{Civilcomments}}}} & Hate Speech &
          Race & 447,998 & \citet{borkan2019nuanced} \\
          \textbf{\textcolor{gray}{\textbf{\textcolor{gray}{MultiNLI}}}} & Entailment & Genre &
          392,702 & \citet{williams2018broad} \\
          \bottomrule
        \end{tabular}
      }
    \end{center}
  }
\end{table*}

\subsection{Datasets}
The corpora supported by ingestion pipelines in \decondtntoolkit fall
into two categories (\tableref{tab:dataset}):
\begin{enumerate}
  \item \textbf{Source Datasets}:
    \textbf{SHAC} \citep{lybarger2021annotating},
    \textbf{MIMIC-Location} \citep{johnson2016mimic},
    and \textbf{HateSpeech}
    \citep{vidgen2021learning, de2018hate}.
    These datasets consist of samples drawn from distinct sources, where
    \emph{provenance} acts as the confounder.

  \item \textbf{\textcolor{gray}{Attribute Datasets}}:
    \textbf{\textcolor{gray}{MultiNLI}} \citep{williams2018broad},
    \textbf{\textcolor{gray}{MIMIC-SubpopBench}}
    \citep{yangChangeHardCloser2023},
    and \textbf{\textcolor{gray}{Civilcomments}}
    \citep{borkan2019nuanced}.
    These datasets comprise samples differentiated by attributes, where
    \emph{attributes} serve as the confounder.
\end{enumerate}
Among the \textbf{Source Datasets},
\textbf{SHAC} and
\textbf{MIMIC-Location} permit exploration of
provenance shift in the context of clinical NLP, the motivating use
case for development of \decondtntoolkit.
On account of a dearth of publicly available multi-institutional
clinical NLP datasets, we have also included publicly available
\textbf{\textcolor{gray}{Attribute Datasets}} which permit a broader
evaluation of robustness to provenance shift and facilitate a
reproducible demonstration of the capabilities of \decondtntoolkit.
In experiments, we explored the extent to which findings related to
provenance shifts in clinical NLP generalize to shifts related to
other attributes in clinical and general domain text.
To demonstrate the current capabilities of the toolkit, we selected a
confounder $Z$ and the outcome $Y$ when multiple options were available.
\appendixref{apd:datasets} provides details on datasets and the
variable selection.

\section{Experimental Setup}
\label{section:setup}
The sections that follow present an evaluation of some of the
algorithms implemented within \decondtntoolkit,
to demonstrate its capabilities and characterize certain aspects of
provenance shift.
We detail the experiment in \appendixref{apd:exp}.

\paragraph{Distribution Manipulation}
Following \citet{landeiro2016robust} and
\citet{veitchCounterfactualInvarianceSpurious2021a},
we introduced spurious correlations between the label $Y$ and the
provenance $Z$ with
a training correlation strength of $\log\alpha^{tr}=-0.6$
(\equationref{eq:alpha})
across all experimental settings\footnote{$P(Y=1 \mid Z=1):P(Y=1 \mid
Z=0) \approx 1:4$}.
To perturb the test distributions under provenance shift,
we generate a suite of test splits where $\log\alpha^{te}$ varies linearly from
$-1$ to $1$.
To isolate the performance impact of subgroup and label shifts,
we enforce the constraints $P^{tr}(Y) = P^{te}(Y)$ and $P^{tr}(Z) = P^{te}(Z)$.
We further ensure that $P(Y)$ and $P(Z)$ follow uniform distributions to ablate
the effects of class and provenance imbalance.
We denote the distribution with $\log\alpha^{te}=-0.6$ as the
in-distribution test set and $\log\alpha^{te}=0.6$ as the
out-of-distribution (OOD) test set, using performance in the OOD as
one measure of robustness to provenance shift.

\paragraph{Hyperparameter Selection}
Following the training protocol in \citet{gulrajaniSearchLostDomain2020}, we
conducted 16 random searches over the joint distribution of hyperparameters
for each algorithm and dataset to ensure a fair ``best-versus-best" comparison.
We list all hyperparameters, their default values, and the joint distribution
for random hyperparameter searching in \tableref{apd:tab:hyperparameters}.
We select the optimal hyperparameters based on the validation
worst-group accuracy (WGA), consistent with the methodology in
\citet{yangChangeHardCloser2023}.

\paragraph{Model Selection}
To estimate the stability and standard deviation of each algorithm, we fix the
optimal hyperparameters and rerun experiments across 5 distinct random seeds.
For each run, we select the model checkpoint that maximizes the worst-group
validation accuracy \citep{yangChangeHardCloser2023}.

\section{Results}
\label{section:results}
\subsection{Learning Dynamics in ERM}
We investigate the learning dynamics of ERM using optimal
hyperparameters to determine:
\begin{enumerate}
  \item Whether neural networks rely on shortcuts to artifacts
    related to provenance (or other confounding variables) during training?
  \item Whether a robust ERM solution can be identified via oracle
    selection on the OOD distribution?
\end{enumerate}
The training progress is standardized to a scale of $[0, 1]$ for all datasets.
\figureref{fig:wga_vs_steps} compares the in-distribution WGA (calculated at
$\log\alpha^{te}=-0.6$) (solid lines) and the OOD WGA (at
$\log\alpha^{te}=0.6$) (dashed lines) throughout the learning process
in \textbf{Source Datasets}.

\paragraph{Neural networks exploit extraneous artifacts for prediction.}
The generalization gap between in-distribution and OOD settings throughout
the learning progress suggests that the predictor $f$ leverages features
in $X$ that are independent of $Y$ but correlated with $Z$ to perform
predictions.
While this achieves high accuracy in-distribution, it results in
limited generalizability to OOD settings (i.e., $f(x) \not\perp Z \mid Y$).
This indicates a consistent shortcut learning pattern across
\textbf{Source Datasets}, characterized by high
in-distribution WGA at the beginning of training.
Conversely, OOD WGA grows slowly or even degrades as training
progresses, despite the fact that OOD settings share identical
marginal distributions $P(Y)$ and $P(Z)$ with the in-distribution settings.

\paragraph{Checkpoint selection does not confer robustness to provenance shift.}
Model selection (selecting an optimal checkpoint during the training
process) is a critical component of the learning pipeline
\citep{gulrajaniSearchLostDomain2020}.
Oracle selection—which assumes access to the test distribution—often
yields over-optimistic results and is generally not considered a
valid benchmarking methodology for real-world deployment.
\figureref{fig:wga_vs_steps} shows WGA (y-axis) as training proceeds (x-axis).
In this context, oracle selection would involve picking the training
step along the x-axis that corresponds to the best WGA.
As can be seen in \figureref{fig:wga_vs_steps}, our results highlight
that even with oracle selection, ERM fails to reach a satisfactory
solution under provenance shift.
The stagnant OOD WGA curves in this figure show that there is no
point in training at which ERM performance in the OOD setting
approaches its in-distribution performance.
This suggests that learning under provenance shift requires more than
just better checkpointing; it necessitates a deliberate
``deconfounding" process to prevent neural networks from using
confounders as indicators of the label.

\begin{figure}[htbp]
  \small
  \floatconts {fig:wga_vs_steps}
  {\caption{In-distribution (solid) and OOD (dashed) WGA for ERM in
      \textbf{Source Datasets}. Y-axis: Worst Group
      Accuracy (WGA). X-axis: normalized progress of training. The gap
      between the dashed and solid lines shows that models make
      inaccurate predictions in the OOD setting throughout their training
  process. }}
  {\includegraphics[width=\linewidth]{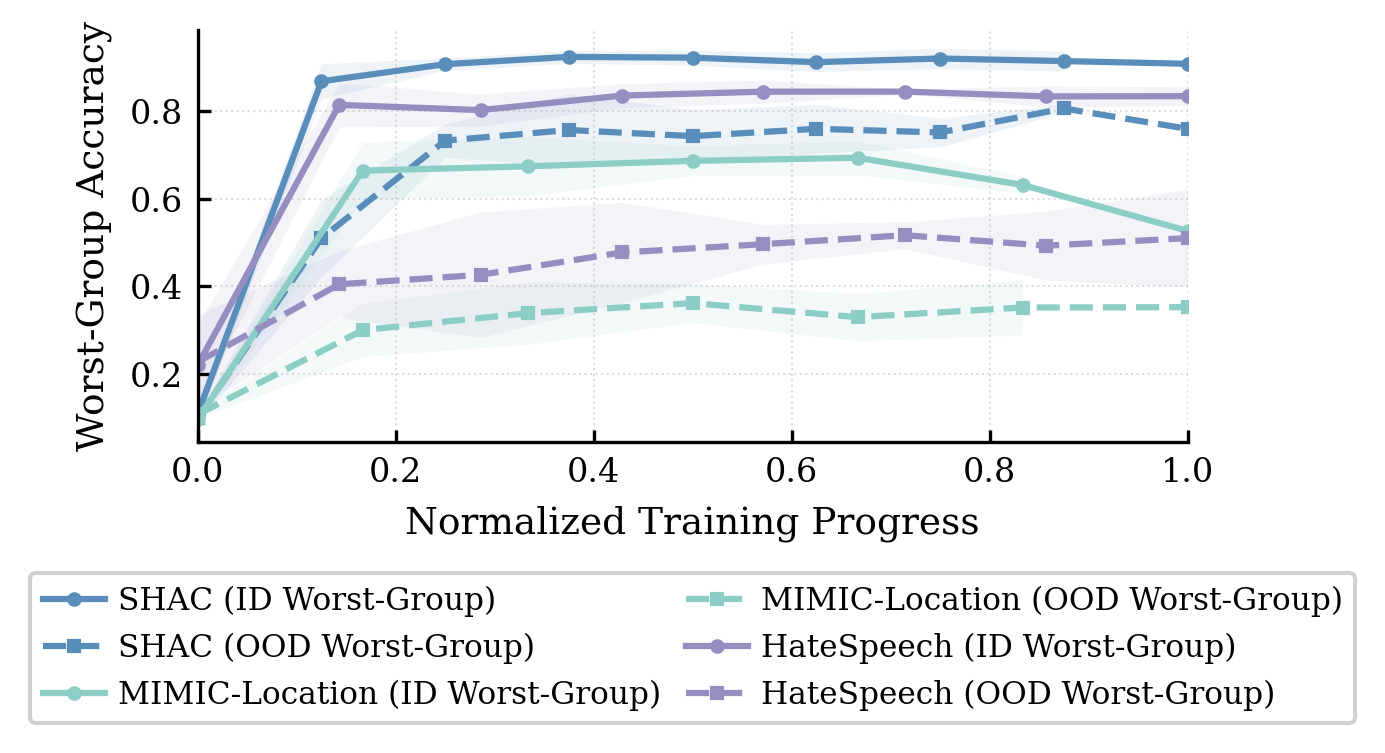}}
\end{figure}

\subsection{Predicting OOD Performance}
Estimating and understanding the predictor $f$ 's performance under
provenance shift is a problem of interest for practitioners.
The ``Accuracy-on-the-line” phenomenon, which has been previously
documented, describes the typical observation of a strong linear
correlation between a model's in-distribution and OOD accuracy
\citep{millerAccuracyLineStrong2021}.
While this is a useful heuristic for model selection (i.e., models
  with higher in-distribution accuracy are likely to have better OOD
accuracy), this phenomenon is not consistent in some OOD benchmarks
\citep{baekAgreementonthelinePredictingPerformance2022}, and has not
been explored in the context of provenance-specific label distribution shifts.
In this section, we examine
\begin{enumerate}
  \item Whether in-distribution WGA is a reliable indicator of OOD
    WGA under provenance shift?
  \item Whether $\alpha^{te}$ can serve as a robust alternative OOD
    performance indicator?
\end{enumerate}
\begin{figure*}[htbp]
  \floatconts
  {fig:wga_indicator}
  {\caption{WGA is not ``on the line" with respect to ID performance,
      but exhibits a
      strong linear relationship with the shift parameter $\alpha$ in
  \textbf{Source Datasets}.}}
  {\includegraphics[width=\textwidth]{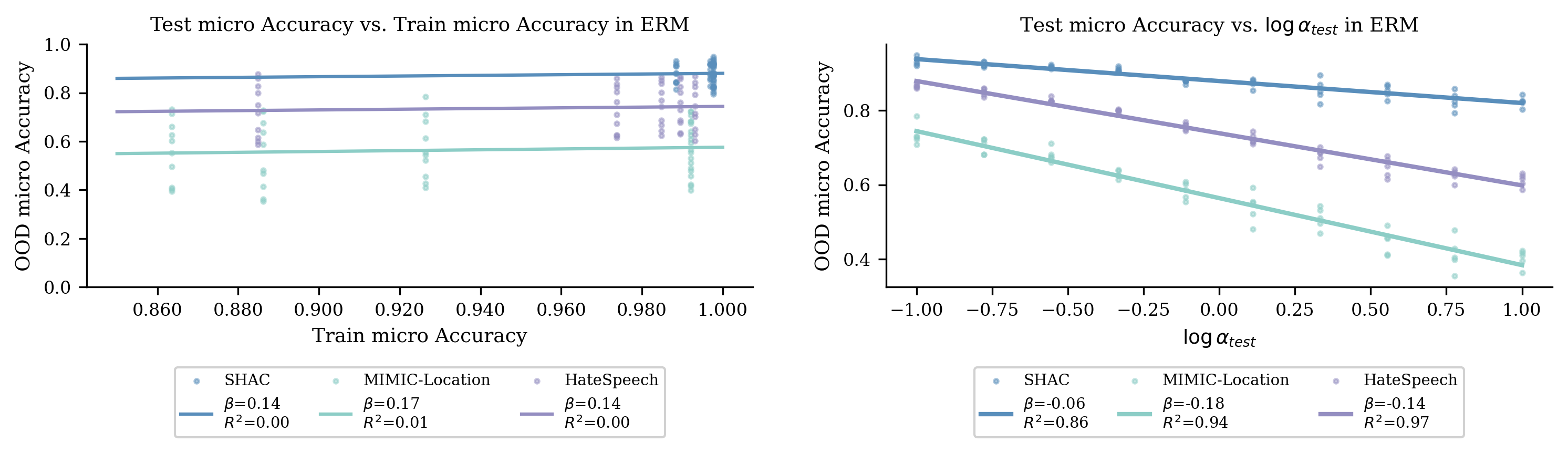}}
\end{figure*}
Using a fixed predictor $f$ per random seed, we calculated the OOD
WGA across a suite of test distributions where $\log\alpha^{te}$
varies linearly from $-1$ to $1$, while models are trained with
$\log\alpha^{tr}=-0.6$.
Results are shown in \figureref{fig:wga_indicator}, which plots the
OOD WGA (y-axis) against the training WGA (left panel) and $\alpha$
(right panel).
Lines have been fit to results for each model across multiple
stochastically-initiated runs of the experiment, each illustrated
with individual markers.
\paragraph{In-distribution WGA can underspecify OOD performance
under provenance shift, while $\alpha$ is a strong OOD performance indicator.}
Similarly to \citet{kohWILDSBenchmarkIntheWild2021}, we did not
observe a strong correlation between in-distribution and OOD WGA
across five runs with different random seeds
(\figureref{fig:wga_indicator}).
Instead, OOD performance varied significantly per run, as evidenced
by the vertical scattering of OOD WGA relative to train WGA.
This indicates that given a train WGA, the OOD WGA may not be
predictable using in-distribution WGA, and improving WGA in
distribution may not directly lead to improved performance in OOD
settings, which is consistent with results reported in
\citep{baekAgreementonthelinePredictingPerformance2022}.
Conversely, we found that $\alpha^{te}$ is strongly correlated with
OOD test WGA across \textbf{Source Datasets},
yielding high $R^2$ values of 0.86, 0.94, and 0.97.
Intuitively, since the predictor $f$ was trained on a distribution
with $\alpha^{tr} = -0.6$, test distributions with an $\alpha^{te}$
closer to $\alpha^{tr}$ generally exhibits higher performance.
These results suggest that practitioners can estimate deployment
robustness by performing a stress test to derive the linear
relationship (coefficient and intercept) between OOD WGA and $\alpha$.

\subsection{Benchmarking Existing Algorithms}
Invariant learning has garnered significant attention over the past decade;
however, its efficacy under provenance shift remains largely unexplored.
In this section, we address two primary questions:
\begin{enumerate}
  \item Can existing algorithms successfully account for an observed
    binary confounder to learn a counterfactually invariant predictor
    $f(X) \perp Z \mid Y$ in OOD settings?
  \item Do invariant learning algorithms remain robust under the
    provenance shift stress tests provided by \decondtntoolkit?
\end{enumerate}
To pursue answers to these questions, we evaluate a suite of
invariant learning algorithms across
\textbf{Source Datasets} in our toolkit.
Following the same protocol as our ERM experiments, we perform a
random search over the joint hyperparameter space for each algorithm
and dataset.
The optimal hyperparameter configurations are detailed in
\tableref{apd:tab:optimalhyper1} and \tableref{apd:tab:optimalhyper2}
for \textbf{Source Datasets}.
The full evaluation results are detailed in \appendixref{apd:results}.
\begin{table*}[h]
  \floatconts{tab:bench}
  {\caption{OOD WGA across algorithms in
  \textbf{Source Datasets}.}}
  {
    \begin{center}
      \begin{tabular}{lcccc}
        \toprule
        \textbf{Algorithm} & \textbf{SHAC} &
        \textbf{MIMIC-Location} &
        \textbf{HateSpeech}  & \textbf{Avg}  \\
        \midrule
        ERM & 0.83 $\pm$ 0.03 & 0.37 $\pm$ 0.06 & 0.51 $\pm$ 0.05 & 0.57 \\
        UpSampling & 0.87 $\pm$ 0.02 & 0.47 $\pm$ 0.04 & 0.58 $\pm$
        0.04 & 0.64 \\
        DownSampling & 0.87 $\pm$ 0.02 & 0.47 $\pm$ 0.05 & 0.61 $\pm$
        0.05 & 0.65 \\
        \midrule
        CAD & 0.85 $\pm$ 0.02 & 0.38 $\pm$ 0.03 & 0.49 $\pm$ 0.04 & 0.57 \\
        CDANN & 0.87 $\pm$ 0.01 & 0.37 $\pm$ 0.02 & 0.45 $\pm$ 0.03 & 0.56 \\
        CORAL & 0.85 $\pm$ 0.05 & 0.40 $\pm$ 0.04 & 0.49 $\pm$ 0.04 & 0.58 \\
        DANN & 0.85 $\pm$ 0.03 & 0.36 $\pm$ 0.03 & 0.52 $\pm$ 0.02 & 0.57 \\
        DFR & 0.85 $\pm$ 0.02 & 0.35 $\pm$ 0.09 & 0.58 $\pm$ 0.04 & 0.60 \\
        DualFilter & 0.82 $\pm$ 0.06 & 0.40 $\pm$ 0.05 & 0.51 $\pm$
        0.04 & 0.58 \\
        Fish & 0.84 $\pm$ 0.03 & 0.36 $\pm$ 0.04 & 0.48 $\pm$ 0.04 & 0.56 \\
        GroupDRO & 0.84 $\pm$ 0.03 & 0.36 $\pm$ 0.09 & 0.50 $\pm$ 0.05 & 0.57 \\
        IRM & 0.85 $\pm$ 0.02 & 0.40 $\pm$ 0.05 & 0.51 $\pm$ 0.03 & 0.59 \\
        JTT & 0.87 $\pm$ 0.01 & 0.38 $\pm$ 0.04 & 0.49 $\pm$ 0.06 & 0.58 \\
        LISA & 0.87 $\pm$ 0.04 & 0.42 $\pm$ 0.05 & 0.56 $\pm$ 0.03 & 0.62 \\
        LfF & 0.77 $\pm$ 0.08 & 0.45 $\pm$ 0.04 & 0.49 $\pm$ 0.06 & 0.57 \\
        MMD & 0.85 $\pm$ 0.02 & 0.38 $\pm$ 0.04 & 0.51 $\pm$ 0.04 & 0.58 \\
        MTL & 0.83 $\pm$ 0.04 & 0.37 $\pm$ 0.06 & 0.47 $\pm$ 0.05 & 0.56 \\
        Mixup & 0.84 $\pm$ 0.03 & 0.37 $\pm$ 0.04 & 0.52 $\pm$ 0.03 & 0.58 \\
        \bottomrule
      \end{tabular}
    \end{center}
  }
\end{table*}

\paragraph{Training distribution manipulation was the most effective
strategy to mitigate confounding under provenance shift.}
Our results indicate that Upsampling and Downsampling consistently improve
OOD WGA across \textbf{Source Datasets}, often
outperforming more complex invariant learning algorithms when
evaluated in a consistent setting
(\tableref{tab:bench}).
This aligns with previous findings by \citet{idrissi2022simple} and
\citet{hendrycksPretrainedTransformersImprove2020}.
Contrary to the prevailing expectation of improvement as data scales,
these results also show that discarding training samples
(\textbf{Downsampling}) can be more effective than
\textbf{Upsampling}, echoing the observations of
\citet{sagawaInvestigationWhyOverparameterization2020}.
Furthermore, algorithms incorporating balanced-sampling
strategies—such as DFR and LISA—consistently outperform those without them.
These observations suggest that while the underlying confounding structure
$X \leftarrow Z \rightarrow Y$ persists, explicitly removing the correlation
between $Z$ and $Y$ in the training distribution is a reliable path
to robustness.

\paragraph{Gains from invariant learning algorithms are dataset-dependent
and marginal compared to ERM.}
Overall, the evaluated algorithms achieve limited robustness to
provenance shift across \textbf{Source Datasets}.
While we observe some benefits from invariant learning methods
(bottom panel of Table \ref{tab:bench}), these improvements are
inconsistent across \textbf{Source Datasets} and
often marginal relative to baseline ERM.
This trend mirrors findings in domain generalization
\citep{gulrajaniSearchLostDomain2020} and subpopulation shift
\citep{yangChangeHardCloser2023}.
Consequently, developing methods to learn invariant features from
joint-imbalanced distributions remains a critical open research
direction under provenance shift.

\subsection{Generalization to \textbf{\textcolor{gray}{Attribute Datasets}}}
In this section, we examine how the above findings generalize to
other attributes in both clinical and non-clinical contexts under
``provenance'' shift.
We repeat all experiments on \textbf{\textcolor{gray}{Attribute Datasets}} using 
the identical pipeline to \textbf{Source Datasets}.
Overall, the characteristics of provenance shift empirically
generalize to other attributes (including sex, race, and genre) in
both clinical and non-clinical context.
We find that the generalization gap between in-distribution and OOD
settings throughout the learning progress persisted, and oracle
selection does not confer robustness to provenance shift in
\textbf{\textcolor{gray}{Attribute Datasets}}
(\figureref{fig:wga_vs_steps_app}).
In addition, in-distribution WGA can underspecify OOD performance
under provenance shift, while $\alpha$ is a strong OOD performance
indicator (\figureref{fig:wga_indicator_app}).
Upsampling and Downsampling consistently improve OOD WGA on
\textbf{\textcolor{gray}{Attribute Datasets}}, often outperforming
more complex invariant learning methods under a controlled evaluation setting.
LISA achieves comparable WGA to subsampling on
\textbf{\textcolor{gray}{Attribute Datasets}}, while other methods
yield only inconsistent and marginal gains over the ERM baseline
(\tableref{apd:tab:wga_app}).

\section{Discussion}
We study the problem of provenance shift and introduce \decondtntoolkit
specializing in the problem setting.
Our comprehensive empirical study using \decondtntoolkit reveals that:
\begin{enumerate}
  \item ERM consistently exploits provenance-related artifacts,
    leading to significant performance degradation in OOD settings.
  \item Oracle selection fails to confer robustness under provenance shift.
    While in-distribution WGA is an unreliable predictor of OOD
    success, the shift parameter $\alpha$ offers a strong linear
    indicator for performance estimation.
  \item Simple training distribution manipulations outperform complex
    invariant learning algorithms, highlighting the ongoing need for
    effective deconfounding methods.
\end{enumerate}
We provide several practical recommendations under provenance shift and
discuss several challenge for future research below:
\paragraph{Recommendation 1: Analyze the causal graph and tackle confounders
at the study design phase.}
Constructing Directed Acyclic Graphs (DAGs) and mitigating
confounding effects are standard practices in fields like
epidemiology, yet research on deliberate modeling of
confounders remains relatively limited within clinical NLP—particularly 
during data collection and ad-hoc modeling.
Consequently, clinical NLP datasets that explicitly account for
confounders are scarce, which impedes research on
confounder adjustment in this context.
Our work demonstrates the failure of ERM under provenance shift, emphasizes
that collecting and accounting for confounders is critical for
robustness in real-world deployment.
Given the current lack of effective post-hoc solutions for unobserved
confounders, we advocate that practitioners analyze the causal graph
during the study design phase, prior to data collection.
This practice can reduce the risk of collection bias, which induce
spurious correlations between the outcome of interest and unobserved
confounders.

\paragraph{Recommendation 2: Stress test the classifier under provenance
shift before online deployment.}
Our study indicates the train-time metrics including WGA may not
be strong predictors of OOD performance in a provenance shift
setting.
Consequently, we recommend that practitioners move beyond static
in-distribution evaluations and adopt a provenance-aware stress
testing protocol before deployment under the scenario of online
learning, in which practitioners need to decide whether to re-train
the model with the incoming online labeled samples.
Practically, practitioners can synthetically vary the provenance
shift parameter ($\alpha^{te}$) to map the model's performance decay
curve using \decondtntoolkit during the ad-hoc testing phase.
Deriving predictive coefficients can also allow for the estimation of
performance in real-time by simply monitoring the $\alpha$ of incoming data.
For example, in scenarios where online sample sizes are too small for
traditional subsampling, stress testing serves as a proxy for
empirical counterfactual invariance.
If a model exhibits a steep performance drop as $\alpha^{te}$
deviates from $\alpha^{tr}$, it indicates a high sensitivity to
observed confounders that ID metrics would otherwise fail to flag.
In this case, we would recommend exploring algorithms that are robust
provenance shift, such as Backdoor Adjustment
\citep{landeiro2016robust,landeiro2018robust,dingBackdoorAdjustmentConfounding2024}
and DAPPER \citep{dingTailoringTaskArithmetic2025}.

\paragraph{Challenge 1: Long-tailed categorical confounders demand a
large sample size to remove spurious correlations in observational data.}
We study a simplified scenario where there is only one observed binary
confounder.
In contrast, consider a categorical confounder $D$ with a set of categories
$D = \{d_1, d_2, \dots, d_k\}$, where the number of categories $k$ is large
(e.g., hundreds of rare comorbidities or specific medication types).
In clinical datasets, the distribution of $D$ is often highly skewed;
while a few categories $d_i$ are frequent, the majority reside in the
``long tail" with minimal representation.
With a such a ``curse of rarity" in categorical confounders, removing spurious
correlations between the outcome $Y$ and the confounder $D$ through
subsampling becomes exponentially more difficult and the effective
sample size is constrained by the rarest category necessary for the analysis.

\paragraph{Challenge 2: Removing confounding effects of continuous variables
is still a open question.}
The current landscape of deconfounding research focuses primarily on
categorical confounders and outcomes.
However, real-world confounders—such as pollution levels, smoking
intensity, or clinical care indices—are frequently continuous.
When dealing with such continuous confounders, some of the
aforementioned methods do not admit a natural extension.
There is a critical need for differentiable deconfounding objectives
that can integrate continuous variable adjustments directly into
neural network loss functions.
Furthermore, the field lacks robust benchmarks that move beyond
binary proxies to incorporate high-fidelity, continuous environmental
and physiological metadata for deconfounding research.

\section{Limitations}
First, we acknowledge that achieving and evaluating counterfactual
invariance are something of an ideal.
With text categorization in particular it may not be possible to
separate the features that indicate $Y$ from those associated with
$Z$ and quantify their separation using observational data.
Empirically, we observe geometric gaps between provenances in the
representation space, which mirrors phenomenons in multi-modal
research \citep{liangMindGapUnderstanding2022} and can not be
trivially eliminated by invariant learning algorithms such as MMD.
Nonetheless, to the extent it is attainable, models with
counterfactual invariance should be robust to provenance shift.
In future work we will incorporate additional methods motivated by
task arithmetic, which were designed to address confounding
adjustment directly \citep{dingTailoringTaskArithmetic2025}, and
quantify the extent of counterfactual invariance using geometric
separations in the representation space.

Second, we restrict our analysis to binary confounders $Z$ and labels $Y$.
Although this serves as an essential starting point, real-world
causal structures involve multi-class or continuous variables;
extending our framework to these richer settings is a key direction
for future work.
In addition, we only explored one direction of shift with a fixed
$\alpha^{tr}$, suggesting the need to explore shift directions for completeness.
The shift direction is fully reversed when the sign of $\alpha$
change, which indicates positive samples comes from one to another.
\citet{dingTailoringTaskArithmetic2025} showed that performance
patterns may vary with the shift direction.
On account of this possibility, we advise simulating shifts in both
directions for a complete characterization of robustness.
Despite the varied relationships between shift directions and
performance, we expect the linear relationship between $\alpha$ and
OOD performance holds when the shift direction is fully reversed
unless data from one source are more difficult to classify.
This is consistent with prior evidence: \citet{landeiro2016robust}
observed a symmetric performance degradation when the shift direction
is fully reversed in the general domain, as have
\citet{dingTailoringTaskArithmetic2025} with clinical data.

Third, our empirical findings are constrained by our controlled
experimental design.
To isolate the effects of provenance shift, we assume fixed marginals
$P^{tr}(Y) = P^{te}(Y)$ and $P^{tr}(Z) = P^{te}(Z)$ with uniform
distributions, which are necessary isolation for a first
characterization of the problem.
Furthermore, our observations are ``sparse", focusing primarily on
the association between $P(Y=0)$ and $P(Z=1)$; these results may not
generalize to other associations or high-dimensional settings.
As noted by \citet{gulrajaniSearchLostDomain2020}, such negative
claims are inherently restricted to the specific settings tested.

Finally, we adopt causal assumptions similar to
\citet{veitchCounterfactualInvarianceSpurious2021a}, requiring that
the invariant component $X^\perp_Z$ is statistically independent of
$Z$. 
In practice, a common cause of $X^\perp_Z$ and $Z$ could
introduce dependencies, potentially leading to an over-conservative
estimate of $\hat{X^\perp_Z}$ and the loss of predictive information.
And a common cause of $X^\perp_Z$ and $Y$ could
introduce dependencies, potentially leading to an over-optimistic
estimate of $\hat{X^\perp_Z}$ and the entanglement of provenance information.
In addition, unobserved confounders $U$ commonly pose a significant challenge
in medical research.
Here we discuss the scenario when $Y \leftarrow U \rightarrow Z$.
In the anti-causal causal graph when $X^\perp_Z \leftarrow Y
\leftarrow U \rightarrow Z$, the learning objective $f(X) \perp Z
\mid Y$ is optimal and risk-invariant, while in the causal graph when
$X^\perp_Z \rightarrow Y \leftarrow U \rightarrow Z$, the learning
objective should be $f(X) \perp Z$ \citep{schrouff2024mind}.
The consideration of $U$ falls outside the scope of the current work
($Y(z) \neq Y(z')$), but it represents an important direction for future work.

\section{Conclusion}
We establish a theoretical link between provenance shift,
counterfactual invariance, and invariant learning, deriving a
learning objective to achieve robustness under provenance shift.
To facilitate empirical study, we introduce \decondtntoolkit, a specialized 
and systematic suite to evaluate and enhance the robustness of text classification 
models to provenance shift.
Our work provides both the theoretical grounding and the practical tools necessary to 
characterize the problem of confounding by provenance.

\bibliography{CHIL2026_PMLR}

\clearpage
\appendix

\section{Proof}
\subsection{Proof of Lemma 2}
\label{apd:proof:lemma2}
\noindent \textbf{Lemma \ref{lemma:decomposition} (Prediction Decomposition).}\\
\textit{
  Let $X^\perp_Y$ be $Y$-invariant components of the input features, such that
  $X^\perp_Y(y) = X^\perp_Y(y')$.
  A discriminative model under the causal graph $\mathcal{G}: X
  \leftarrow Z \rightarrow Y$
  can be decomposed into two components: the inference and the
  provenance mechanism.
  \begin{equation}
    P(Y \mid X^\perp_Y)
    = \int
    \underbrace{P(Y \mid Z)}_{\substack{\text{provenance}}}
    \,
    \underbrace{P(Z \mid X^\perp_Y)}_{\substack{\text{inference}}}
    \, dZ
  \end{equation}
}
\begin{proof}
  By the law of total probability, we marginalize over the $Z$:
  \begin{equation*}
    P(Y \mid X^\perp_Y) = \int P(Y \mid Z, X^\perp_Y) P(Z \mid X^\perp_Y) \, dZ
  \end{equation*}
  By definition, $X^\perp_Y$ is $Y$-invariant, implying $P(Y \mid Z,
  X^\perp_Y) = P(Y \mid Z)$.
  Therefore, the conditional probability simplifies to
  $P(Y \mid Z, X^\perp_Y) = P(Y \mid Z)$.
  Substituting this back into the integral yields:
  \begin{align*}
    P(Y \mid X^\perp_Y) &= \int \underbrace{P(Y \mid
    Z)}_{\text{provenance}} \underbrace{P(Z \mid
    X^\perp_Y)}_{\text{inference}} \, dZ
  \end{align*}
\end{proof}

\subsection{Proof of Proposition 4}
\label{apd:proof:proposition4}
\noindent \textbf{Proposition \ref{proposition:robustness}
(Provenance Robustness).}\\
\textit{
  If a predictor $f: \mathcal{X} \rightarrow \mathcal{Y}$ satisfies
  counterfactual invariance such that
  $f(X(z)) = f(X(z'))$ for all $z, z' \in \mathcal{Z}$, then the
  predictor is robust to provenance shift.
  Specifically, the risk $\mathbb{E}[\mathcal{L}(f(X), Y)]$ remains
  constant under any intervention on
  the mechanism $P(Y|Z)$, provided $P(Y)$ remains invariant.
  This robustness holds under both:
  \begin{enumerate}
    \item Anti-causal settings ($Y \rightarrow X$);
    \item Causal settings ($X \rightarrow Y$): Provided that $Y \perp
      X \mid \{X^\perp_Z, Z\}$
      and the label satisfies counterfactual consistency, $Y(z) =
      Y(z')$ for all $z, z' \in \mathcal{Z}$.
  \end{enumerate}
}

\begin{proof}
  We marginalize the probability of a prediction $\hat{y}$ over the
  joint distribution of $Y$ and $Z$:
  \scriptsize
  \begin{align}
    P(f(X)=\hat{y}) &= \int_{\mathcal{Y}} \int_{\mathcal{Z}}
    P(f(X)=\hat{y} \mid Y, Z) P(Y, Z) \, dZ \, dY \\
    &= \int_{\mathcal{Y}} \int_{\mathcal{Z}} P(f(X)=\hat{y} \mid Y,
    Z) P(Y \mid Z) P(Z) \, dZ \, dY
  \end{align}
  \normalsize

  \begin{theorem}[\footnotesize Counterfactual Invariant Predictor\normalsize]
    (Restated from Thereom 3.2 in
    \citet{veitchCounterfactualInvarianceSpurious2021a})
    If $f$ is a counterfactual invariant predictor,
    \begin{enumerate}
      \item Under the anti-causal graph, then $f(X) \perp Z \mid Y$.
      \item Under the causal-direction graph, if $Y$ and $Z$ are not
        subject to selection
        (but possibly confounded), then $f(X) \perp Z$.
      \item Under the causal-direction graph, if $Y \perp X \mid
        \{X^\perp_Z, Z\}$ and
        $Y(z) = Y(z')$ for all $z, z' \in \mathcal{Z}$, then $f(X)
        \perp Z \mid Y$.
    \end{enumerate}
  \end{theorem}

  Applying the conditional independence property $f(X) \perp Z \mid
  Y$, which holds in the above two causal graphs,
  the term $P(f(X)=\hat{y} \mid Y, Z)$
  simplifies to $P(f(X)=\hat{y} \mid Y)$:
  \scriptsize
  \begin{align}
    P(f(X)=\hat{y}) &= \int_{\mathcal{Y}} P(f(X)=\hat{y} \mid Y)
    \left( \int_{\mathcal{Z}} P(Y \mid Z) P(Z) \, dZ \right) dY \\
    &= \int_{\mathcal{Y}} P(f(X)=\hat{y} \mid Y) P(Y) \, dY
  \end{align}
  \normalsize

  The final expression shows that the distribution of predictions
  (and consequently the risk) depends only on
  the model performance $P(\hat{Y} \mid Y)$ and the label prior $P(Y)$.
  Since the provenance mechanism $P(Y|Z)$
  has been marginalized out, the predictor is robust to shifts in
  that mechanism.
\end{proof}

\section{Related Work}
\label{apd:relatedwork} In this section, we provide an exhaustive
literature review
on possible solutions to provenance shift from different perspectives.

\subsection{Distribution Alignment}
\label{apd:distributionalignment}
Distribution alignment techniques align high-level statistics (e.g., mean) of
distance measurement on features across provenances by minimizing their
differences.
For example, the Maximum Mean Discrepancy (MMD) measures and then minimizes
the probability divergence between two distributions by mapping samples to a
reproducing kernel Hilbert space (RKHS, e.g., using Gaussian kernels) and then
deriving the difference of distribution across provenances
\citep{liDomainGeneralizationAdversarial2018}.
Similarly, \citet{sunDeepCORALCorrelation2016} proposed deep CORAL, which
aligns the mean and the variance of the features instead of using parametric
Gaussian kernels as in MMD. From a causal perspective,
\citet{veitchCounterfactualInvarianceSpurious2021a}
theoretically proves that domain alignment should consider the causal direction
between the data and label. Specifically, when modeling causal
direction data (i.e.,
the data causes the label), features from different provenances should be
aligned like the classical practice.
However, when modeling anti-casual direction data (i.e., the label
causes the data),
features from different provenances should be aligned \emph{conditioned on the
label}\footnote{Here, we use the MMD, denoted as $MMD$, for alignment
  as an example.
  Formally, if $X\rightarrow Y$, then the training objective is
  $\argmin_{X}\mathrm{MMD}(X)$. If $X\leftarrow Y$, then the training
  objective is
$\argmin_{X}\mathrm{MMD}(X \mid Z)$.}

\subsection{Adversarial Training}
\label{apd:adversarialtraining}
Compared to distribution alignment techniques, which make
low-dimensional statistics
indiscriminate to provenance, adversarial training techniques make the full
feature space indiscriminate to provenance using a parametric provenance
discriminator (i.e., predictor) based on neural networks.
Adversarial training was initially proposed to train a generative model, which
generates photorealistic images using random noise
\citep{goodfellow2014generative}.
Specifically, adversarial training resembles a two-player game, which trains a
\emph{discriminator} to distinguish between real and the generated fake images
by minimizing the real-fake classification loss, while rewarding the
\emph{generator}
to fool the discriminator by maximizing the real-fake binary
classification loss.
With this training paradigm, the generator will finally generate photorealistic
images that can fool a strong real-fake image discriminator.
\citet{ganinDomainAdversarialTrainingNeural2016} extends this idea to
domain adversarial
training by learning domain-agnostic features that can confuse a
domain discriminator.
Specifically, \citet{ganinDomainAdversarialTrainingNeural2016} proposed the
Domain Adversarial Neural Network to learn features discriminative for the
main learning task but indiscriminate with respect to provenances.
\citet{liDeepDomainGeneralization2018} proposed Conditional DANN (CDANN) to
adapt to concept shift by making the domain discriminator predict the
permutation of provenance and label. In this case, the domain discriminator
can't tell the features from different provenances but with the same label
apart.
\citet{ruanOptimalRepresentationsCovariate2021} extends the domain adversarial
training objective and proposes the Contrastive Adversarial Domain (CAD)
objective by explicitly considering the discriminative capability of features.
Specifically, the optimal features should remain discriminative for the
learning task while distributions of features are indiscriminative to
provenances in the training objective.

Similarly, adversarial training on word embedding makes the word
embedding indiscriminative
to the provenance in a post-hoc way.
For example, Iterative Null-space Projection (INLP) iteratively
removes an attribute
from the word embedding space by projecting it to an attribute-agnostic subspace
\citep{ravfogelNullItOut2020}.
Specifically, it first trains a linear provenance discriminator and then
projects the feature on the null-space of the provenance discriminator, which is
expected to be non-discriminative of the provenance.

Though adversarial training is a strong method to make the feature
provenance-agnostic,
it can increase the reliance on spurious features instead of core features in
a neural network with regulation (e.g. $\mathcal{l}_{2}$ regulation)
\citep{moayeriExplicitTradeoffsAdversarial2022}.
Besides, adversarial training can have an unintended consequence in reducing
robustness to distribution shift, specifically when spurious
correlations are changed
in the test distribution.
\citet{kumarProbingClassifiersAre2022} argue that both adversarial training and
word embedding editing methods can be counter-productive in real-world settings
where the label is naturally correlated with the provenance, because they
internally use an auxiliary (or probing) provenance classifier based on the
features learnt by the main-task classifier.
Specifically, the provenance classifier cannot be a reliable signal on whether
the feature is causally derived from the provenance.

\subsection{Invariant Learning}
\label{apd:invarianttraining}
The Invariant Risk Minimization (IRM) objective is a variant of the ERM, which
finds the features such that the optimal main-task classifier based
on these features
is \textbf{simultaneously} optimal for all provenances
To overcome the challenging and bi-level optimization problem,
\citet{arjovskyInvariantRiskMinimization2020}
proposed a practical version, IRMv1, which penalizes the risk variation across
provenances using one single scalable parameter.
\citet{drankerIRMWhenIt2021} conducts a case study in natural language
inference and finds that learning complex features in place of spurious
correlation proves to be difficult in the wild, leading to little
incremental over
ERM.
In addition, the performance of IRM depends on the sample size, the prevalence
of spurious correlation, and the strength of spurious correlation, therefore
limiting IRM's advantage in the wild.
\citet{ahujaInvariancePrincipleMeets2021} adds information bottleneck
constraints \citep{tishby2000information} to the IRM to mitigate failures when
invariant features capture all the information about the label (e.g., the
label is a deterministic function of the invariant feature).
An empirical study shows that the Information Bottleneck-based IRM achieves
consistent calibration across provenances and information compression
techniques (e.g., information bottleneck) are potentially effective in
achieving invariance \citep{yoshidaUnderstandingVariantsInvariant2024}.

\subsection{Meta Learning}
Unlike conventional machine learning approaches that use a \emph{fixed}
learning algorithm to solve the question from scratch, meta learning, known as
learning-to-learning, aims to improve the learning algorithm itself
using experience
of multiple learning episodes \citep{hospedalesMetaLearningNeuralNetworks2022}.
The motivation behind meta learning for distribution shift is that
exposing a model
to distribution shift can make it learn to adapt to the shifted domain.
For example, the Meta-Learning Domain Generalization approach (MLDG) simulates
distribution shift by synthetically perturbing the pseudo testing distribution
in each training step \citep{liLearningGeneralizeMetaLearning2018}.
The meta-optimization objective is that the performance on the
training distribution
should also improve with the pseudo testing distribution simultaneously.

\subsection{Gradient Matching}
\label{apd:gradientmatching}
Gradients are signals that control learning orientation in neural
network training.
The idea of the gradient matching method under provenance shift is to align
the provenance-specific gradient directions to derive provenance-invariant
optimization paths and features.
\citet{shiGradientMatchingDomain2021} propose an Inter-Domain Gradient
Matching (IDGM) objective that maximizes the inner product between gradients
of different provenance.
To overcome the computational cost of the second-order derivatives,
\citet{shiGradientMatchingDomain2021}
propose Fish, a meta-learning optimization method based on simply first-order
derivatives, to optimize the IDGM objective.
\citet{rameFishrInvariantGradient2022} extends the IDGM objective to the Fishr
regularization, in which they match the provenance-specific gradient variances
to match provenance-level risks and Hessians.

\subsection{Distributionally Robust Optimization}
\label{apd:dro}
Group Distributionally Robust Optimization (Group DRO) is a principle that
minimizes the \emph{worst-case} over potential test distributions training error
instead of the \emph{average} training error in classical ERM
\citep{ben-talRobustSolutionsOptimization2013}.
Overparameterization refers to increasing model size beyond the point of zero
training error, which can exacerbate spurious correlations when they are present
in the training data \citep{sagawaInvestigationWhyOverparameterization2020}.
\citet{sagawa*DistributionallyRobustNeural2019} study group DRO in the context
of overparameterized neural networks, which achieve zero training
error but don't
generalize to the worst group at testing time.
They found that strongly-regularized group DRO models without vanishing
training signals have good worst-case performance.
Besides, \citet{sagawa*DistributionallyRobustNeural2019} developed a stochastic,
online, and greedy algorithm for group DRO optimization that can scale to large
models and datasets.
\citet{zhouExaminingCombatingSpurious2021} propose Group Conditional
DRO (GC-DRO)
to introduce group uncertainty in the case when pre-defined group
information does
not directly account for various spurious correlations.
Specifically, GC-DRO introduces group-level and group-conditioned sample-level
weights in the training process to generate a more flexible
uncertainty set compared
to group DRO, which treats each group as a unit and thus removes sample-level
weights \emph{within each group}.
Inspired by the Group DRO, \citet{eastwoodProbableDomainGeneralization2022}
propose Quantile Risk Minimization (QRM) objective, which seek predictors that
perform well with a specific probability $\alpha$.
In other words, while ERM seeks predictors that perform well on the
\emph{average-case}
and Group DRO seeks predictors that perform well on the \emph{worst-case}, QRM
seeks \emph{probabilistic} predictors that perform well with
probability $\alpha$
by minimizing the $\alpha$-quantile of the estimated risk
distribution over training
domains.
Specifically, \citet{eastwoodProbableDomainGeneralization2022} proposed the
Empirical QRM (EQRM), which leverages kernel density estimation (KDE,
\citet{parzenEstimationProbabilityDensity1962})
to estimate the cumulative distribution function (CDF) of the training risk and
minimize the probability at the $\alpha$-quantile of the CDF.

\subsection{Double-Stage Training}
\label{apd:doublestage}
Double-phase training strategies add a fine-tuning step after the standard ERM
training.
Just Train Twice (JTT) retrains the neural network using a reweighted dataset
(second stage), in which samples misclassified at the end of a few steps of the
standard ERM (first stage) are upweighted \citep{liuJustTrainTwice2021}.
Different from most mentioned methods, JTT does not require group information
during training time.
Intuitively, JTT improved the worst-group performance by focusing on samples
from groups where standard ERM models perform poorly.
\cite{kirichenkoLastLayerReTraining2023} observed that ERM can learn
core features
even when spurious correlations are present and are much simpler than the core
features in the training data.
In addition, they noticed the final classification layer of the model
highly weights
the spurious features, resulting in poor predictions on the minority groups.
Therefore, \cite{kirichenkoLastLayerReTraining2023} proposed Deep Feature
Reweighting (DFR), which leveraged a small set of data where the
spurious correlation
does not hold to retrain the classifier exclusively after the
standard ERM training.
Following DFR, \citet{chenProjectProbeSampleEfficient2023} projects the
features in an orthogonal space and then trains the domain-specific
classifiers using small data sets from the target domain.
This approach is theoretically sample-efficient in training the domain-specific
classifier with minimal distribution assumptions.

\subsection{Re-Sampling and Re-Weighting}
\label{apd:sampling}
Re-sampling and re-weighting strategies manipulate the training distribution to
prevent the model from learning from spurious correlation.
Re-sampling simply down-samples and re-weighting simply up-weights
the sample according to the number of samples per class or group
\citep{japkowicz2000class}.
In other words, re-sample throws away data points until the classes
or groups are balanced in size, followed by ERM on the down-sampled dataset.
While discarding data opposes common wisdom in learning theory, where
the expected error is inversely proportional to the sample size,
\citet{chaudhuriWhyDoesThrowing2023} theoretically shows that minor
groups (i.e., the tail of the data distribution) play an important
role in determining the worst-group-accuracy of a classifier that
classifies samples with the largest possible gap between the labels.
In addition, resampling outperforms ERM in worst-group-error when learning
from imbalanced classes with tails and balanced classes but imbalanced groups.
In the previous work, \citet{gulrajaniSearchLostDomain2020}
implements a group-balanced ERM and found that it outperforms the
state-of-the-art in terms of average performance across datasets, and
the improvement of existing algorithms is too incremental compared to
the group-balanced ERM.
This finding is consolidated by
\citet{hendrycksPretrainedTransformersImprove2020}, who found that if
the samples of different datasets were unbiasedly drawn from the same
distribution, the model should not discover any dataset-specific patterns.
Similarly, \citet{sagawaInvestigationWhyOverparameterization2020}
finds that subsampling the majority group can empirically achieve low
minority error in the overparameterized regime, but upweighting the
minority group fails.
\citet{idrissiSimpleDataBalancing2022} extends the work by considering both
class-balance and group-balance and finds that re-sampling and re-weighting
are competitive in common benchmarks in which classes or groups are imbalanced.
Re-sampling and re-weighting are advanced and are recommended in the
wild as they are faster to train and are hyper-parameter-free.
\citet{cohen-wangAskYourDistribution2024} study pre-trained models
and find that fine-tuning on a small but balanced dataset can result
in significantly more robust models than fine-tuning on a large but
imbalanced dataset.

\subsection{Domain Interpolation}
\label{apd:domaininterpolation}
Mixup linearly combines, or \emph{interpolates}, two random samples and
applies the same interpolation strategy on the corresponding labels to create
neighborhood (or \emph{vicinity}) samples, thereby overcoming the limitations
of dataset-dependent and intra-label data augmentation
\citep{zhangMixupEmpiricalRisk2018a}.
\citet{yanImproveUnsupervisedDomain2020} extended this idea to
inter-domain Mixup,
which linearly combines two random samples from different domains by
explicitly considering the sample provenance.
In a case study of natural language understanding tasks using
inter-domain Mixup,
this method was empirically demonstrated to improve the generalizability of
minor groups across encoder, encoder-decoder, and decoder-only architectures
consistently \citep{korakakisMitigatingShortcutLearning2025}.
Following inter-domain Mixup, Learning Invariant Predictors with
Selective Augmentation
(LISA) combines inter-label Mixup and inter-domain Mixup by randomly selecting
one strategy for augmentation to introduce cross-label augmentation
to inter-domain
Mixup \citep{yaoImprovingOutofDistributionRobustness2022}.
Although empirical results of LISA have shown remarkable improvements in several
benchmarks, \citet{teneySelectiveMixupHelps2024} pointed out that the
performance
gain of LISA might originate from implicitly balancing the training distribution
instead of interpolated mixing.
For example, intra-label Mixup implicitly resamples the training data
to a class-uniform
distribution, which is a strategy in tackling label shift.

\section{Algorithms}
\label{apd:algorithms}

In this section, we denote $\mathcal{X}$ as the input space,
$\mathcal{Y}$ as the label space,
and $\mathcal{H}$ as the representation space.
Let $\Phi: \mathcal{X} \rightarrow \mathcal{H}$ represent the featurizer and
$w: \mathcal{H} \rightarrow \mathcal{Y}$ the classifier,
such that $f = w \circ \Phi$ defines the composite predictor.
The loss function is denoted by $\mathcal{L}: \mathcal{Y} \times
\mathcal{Y} \rightarrow \mathbb{R}_{\geq 0}$.

\paragraph{ERM}
ERM serves as the standard baseline for supervised learning.
Given a dataset $\mathcal{D} = \{(x_i, y_i)\}_{i=1}^n$, the learning
objective is to find
a predictor $f \in \mathcal{F}$ that minimizes the empirical risk:
\begin{equation}
  \min_{f} \mathbb{E}_{(x,y) \sim P^{tr}} [\mathcal{L}(f(x), y)]
\end{equation}

\paragraph{UpSampling and DownSampling \citep{japkowicz2000class}}
UpSampling and DownSampling share the ERM objective but modify the
training distribution $P^{tr}$ to produce a
balanced joint distribution $\hat{P}(X, Y, Z)$.
Specifically, the target distribution is adjusted such that:
\begin{equation}
  \hat{P}(X,Y,Z) = P(X,Y,Z) \frac{P(Y)P(Z)}{P(Y,Z)}
\end{equation}
This adjustment aims to decouple the dependence between labels $Y$
and confounder $Z$.
A detailed discussion of specific sampling implementations is
provided in Appendix \ref{apd:sampling}.
\paragraph{DANN and CDANN
  \citep{ganinDomainAdversarialTrainingNeural2016,
liDeepDomainGeneralization2018}}
DANN introduce a domain discriminator $f_{disc}: \mathcal{H}
\rightarrow \mathcal{Z}$ to ensure the representations
are uninformative of the domain $z \in \mathcal{Z}$.
The objective is:
\begin{equation}
  \begin{aligned}
    \min_{w, \Phi}& \max_{f_{disc}} \mathbb{E}_{(x,y) \sim P^{tr}}
    [\mathcal{L}(w(\Phi(x)), y)]\\
    &- \alpha \mathbb{E}_{(x,z) \sim P^{tr}}
    [\mathcal{L}_{disc}(f_{disc}(\Phi(x)), z)]
  \end{aligned}
\end{equation}

where $\mathcal{L}_{disc}$ is the binary cross-entropy loss for
domain classification
and $\alpha$ is a trade-off hyperparameter.

CDANN conditions the domain discriminator by adding a label-specific
embedding to the
provenance discriminator input.
Denote $\mathbf{e}_y$ as the embedding vector associated with label $y$.
The objective of CDANN is:
\begin{equation}
  \begin{aligned}
    \min_{w, \Phi}& \max_{f_{disc}} \mathbb{E}_{(x,y) \sim P^{tr}}
    [\mathcal{L}(w(\Phi(x)), y)]\\
    &- \alpha \mathbb{E}_{(x,y,z) \sim P^{tr}}
    [\mathcal{L}_{disc}(f_{disc}(\Phi(x), \mathbf{e}_y), z)]
  \end{aligned}
\end{equation}

A discussion of adversarial techniques is provided in Appendix
\ref{apd:adversarialtraining}.

\paragraph{MMD and CORAL
\citep{sunDeepCORALCorrelation2016,liDomainGeneralizationAdversarial2018}}
CORAL and MMD penalize the distribution distances across provenances.
CORAL minimizes the distance between the second-order statistics of
the representations across provenances.
Let $C$ be the covariance matrices of the source and target features
in $\mathcal{H}$.
The CORAL penalty is defined using the squared Frobenius norm:
\begin{equation}
  \mathcal{L}_{\mathrm{CORAL}}=\sum_{i,j \in \mathcal{Z}}
  \frac{1}{4d^2} \| C_i - C_j \|_F^2
\end{equation}
where $d$ is the dimension of the representations.

MMD is a non-parametric metric that measures the discrepancy between
two distributions $P_s$ and $P_t$ by
comparing their mean embeddings in an RKHS $\mathcal{F}$ associated
with a kernel $k(\cdot, \cdot)$.
The squared MMD distance is:
\begin{equation}
  \mathcal{L}_{MMD} = \sum_{i,j \in \mathcal{Z}}\left\| \mathbb{E}_{x
  \sim P_i}[\phi(x)] - \mathbb{E}_{u \sim P_j}[\phi(u)] \right\|_{\mathcal{F}}^2
\end{equation}
A discussion of domain alignment techniques is provided in Appendix
\ref{apd:distributionalignment}.
\paragraph{CAD \citep{ruanOptimalRepresentationsCovariate2021}}
The CAD learns to keep the representation $H$ discriminative for the
learning task and maintain the support of
its marginal distribution invariant to shifts.
Denote $A$ is a random variable sampled conditionally from $X$.
CAD instead maximizes the mutual information $I[A;\Phi(X)]$ based on
InfoNCE \citep{oord2018representation} as an
alternative learning objective.
In addition, CAD introduces a domain bottleneck $I[\Phi(X);Z]$, which
enforces support match using a KL divergence.
\paragraph{Mixup and LISA \citep{zhangMixupEmpiricalRisk2018a,
yaoImprovingOutofDistributionRobustness2022}}
Mixup linearly interpolates one sample $(x_i, y_i)$ with another
random sample $(x_j, y_j)$ from the training data:
\begin{equation}
  \min_{f} \mathbb{E}_{(x,y) \sim P^{tr}} \mathcal{L}(f(\hat{x}), \hat{y})
\end{equation}
where $\hat{x}=\lambda x_i + (1-\lambda)x_j$, $\hat{y}=\lambda
\hat{y}_i + (1-\lambda)\hat{x}_j$, and $\lambda \sim
\text{Beta}(\alpha, \alpha)$, for $\alpha \in (0, \infty)$.

LISA explicitly consider label and provenance in the Mixup process by
mixing samples (1) with the same provenance but
different labels (i.e., intra-domain $y_i \neq y_j$ and $z_i = z_j$)
and (2) with the same provenance but
different labels (i.e., intra-label $y_i = y_j$ and $z_i \neq z_j$).
A discussion of domain interpolation techniques is provided in
\appendixref{apd:domaininterpolation}.

\paragraph{IRM \citep{arjovskyInvariantRiskMinimization2020}}
IRM penalizes feature distributions that have different optimal
linear classifiers for each domain.
The IRM objective is
\begin{equation}
  \begin{aligned}
    \min_{\substack{\Phi: \mathcal{X} \rightarrow \mathcal{H} \\ w:
    \mathcal{H} \rightarrow \mathcal{Y}}}
    &\sum_{z \in \mathcal{Z}}\mathbb{E}_{(x,y) \sim P_z}
    [\mathcal{L}(w \circ \Phi(x), y)] \\
    \text{s.t.} \
    \min_{\bar{w}:\mathcal{H} \rightarrow \mathcal{Y}}
    &\mathbb{E}_{(x,y) \sim P_z} [\mathcal{L}(\bar{w} \circ \Phi(x), y)],
    \forall z \in \mathcal{Z} .
  \end{aligned}
\end{equation}

A discussion of invariant learning techniques is provided in
\appendixref{apd:invarianttraining}.

\paragraph{GroupDRO \citep{sagawa*DistributionallyRobustNeural2019}}
GroupDRO uses distributionally robust optimization to explicitly
minimize the loss
on the worst-case domain during training:
\begin{equation}
  \min_{f} \max_{z \in \mathcal{Z}} \mathbb{E}_{(x,y) \sim P^{tr}_z}
  [\mathcal{L}(f(x), y)]
\end{equation}
In practice, this is implemented using an online estimation of group
weights $q$,
where $q$ is updated via exponentiated gradients to shift the model's
focus toward groups with higher training error.
A discussion of distributionally robust techniques is provided in
\appendixref{apd:dro}.

\paragraph{Fish \citep{shiGradientMatchingDomain2021}}
Fish proposed an inter-domain gradient matching objective to align
the gradient direction across provenances.
To reduce the computing complexity of second-order derivatives, Fish
use a first-order algorithm for approximation.
Similar to meta learning, Fish updates a clone $\hat{f}$ using data
per provenance and then update using a
weighted difference bwteen the clone model $\hat{f}$ and the modek
before update $f$.
A discussion of gradient matching techniques is provided in
\appendixref{apd:gradientmatching}.

\paragraph{MTL \citep{blanchard2021domain}}
MTL augments the representation with the marginal distribution of
feature vectors.
Specifically, MTL maintains a domain embedding $\mathbf{e}$, which is
the empirical mean of
the representations in that domain:
\begin{equation}
  \mathbf{e}_z = \mathbb{E}_{x \sim P^{tr}_z} [\Phi(x)]
\end{equation}
In practice, $\mathbf{e}_z$ is updated via an exponential moving
average (EMA) during training.
The MTL learning objective is based on the concatenation of the
point-wise feature with the domain embedding:
\begin{equation}
  \min_{f}\mathbb{E}_{(x,y) \sim P^{tr}_z}
  [\mathcal{L}(w\left(\left\|\Phi(x); \mathbf{e}_z\right\|_2\right), y)]
\end{equation}
where $[\cdot ; \cdot]$ denotes concatenation.
In inference, $\mathbf{e}_z$ degrades to the mean of $\Phi(x)$.
This allows the classifier to adapt its decision boundary based on
the global statistics of the
current marginal distribution.

\paragraph{LfF \citep{nam2020learning}}
LfF simultaneously trains a biased classifier $f_{B}$ and a debiased
classifier $f_{D}$.
The biased model is optimized by $\mathcal{L}_{\mathrm{GCE}}$ to amplify bias:
\begin{equation}
  \mathcal{L}_{\mathrm{GCE}} = \mathbb{E}_{(x,y) \sim P^{tr}}
  \frac{1-P\left( f(x) = y \right)^q}{q}
\end{equation}
where $q\in (0, 1]$ is a hyperparameter that controls the amplification degree.
The debiased model is optimized by
\small
\begin{equation}
  \mathcal{L}_{\mathrm{LfF}} = \mathbb{E}_{(x,y) \sim P^{tr}}
  \frac{\mathcal{L}(f_{B}(x), y)}{\mathcal{L}(f_{B}(x), y) +
  \mathcal{L}(f_{D}(x), y)}\mathcal{L}(f_{D}(x), y)
\end{equation}
\normalsize

\paragraph{JTT \citep{liuJustTrainTwice2021}}
JTT first curates an error set $E$ of training samples that the ERM
model $f_{\mathrm{ERM}}$ misclassifies:
\begin{equation}
  E = \left\{ (x_i, y_i) \;\text{s.t.}\; f_{\mathrm{ERM}}(x_i) \neq y_i \right\}
\end{equation}
Then JTT upweight the samples in $E$:
\begin{equation}
  \min_{f} \lambda \sum_{(x,y) \in E} \mathcal{L}(f(x), y)
  + \sum_{(x,y) \notin E} \mathcal{L}(f(x), y)
\end{equation}
where $\lambda$ is a hyperparameter.

\paragraph{DFR \citep{kirichenkoLastLayerReTraining2023}}
DFR adapts the standardized ERM training procedure.
It then freezes the featurizer $\Phi$ and and retrains the classifier
$w$ with $l_2$ penalization using a subset from $\mathcal{D}^{tr}$.
The subset is sampled to follow a
joint balanced distribution $\hat{P^{tr}}$.

\paragraph{BackDoor \citep{dingBackdoorAdjustmentConfounding2024}}
BackDoor includes a provenance embedding in the model to conduct
backdoor adjustment.
The prediction can be decomposed as
\begin{equation}
  P(Y \mid X) = \sum_{z \in \mathcal{Z}} P(Y \mid X,z)P(z)
\end{equation}
Backdoor uses $P^{tr}(Z)$ as $P(Z)$ in inference.

\paragraph{DualFilter \citep{shengMitigatingConfoundingSpeechBased2025}}
DualFilter first trains a task classifier from the pretrained weights
$\theta_0$, obtaining the
accumulative weight changes $\Delta_{task}$ across the network and
finetuned weights $\hat \theta$; and then it
trains a provenance classifier $g$ from the same pretrained
checkpoint with accumulative weight change
$\Delta_{prov}$.
For both $\Delta$, pick the top $k$ most changed weights locations
for set operation and masking.
\begin{equation}
  \begin{aligned}
    &\text{Let } M = \Delta_{task, k} \odot \Delta_{prov, k} \\
    &f_{\theta'}: \theta_i' \leftarrow \hat \theta_i = 0 \quad \forall i \in M,
  \end{aligned}
\end{equation}
where $\odot$ is an arbitrary set operation and $f_{\theta'}$ is the
masked task model ready for inference.

\section{Datasets}
\label{apd:datasets}
\paragraph{\textbf{SHAC} \citep{lybarger2021annotating}}
The Social History Annotation Corpus (SHAC) consists of clinical
notes from two institutions:
the University of Washington Medical Center and MIMIC-III.
The primary task is to identify information regarding substance use
from clinical text.
In our experiments, we define the prediction label as the presence of
drug abuse and
utilize the data source (institution) as the provenance.
\paragraph{\textbf{MIMIC-Location} \citep{johnson2016mimic}}
\textbf{MIMIC-Location} contains clinical notes
recorded during
the first 48 hours of a hospital stay, sourced from the MIMIC-III database.
We define the prediction target as in-hospital mortality and use
admission location
(Emergency Room Admission vs. Physician Referral / Normal Delivery)
as the provenance.

\paragraph{\textbf{HateSpeech} \citep{vidgen2021learning,
de2018hate}}
We utilize a hate speech detection dataset curated by
\citet{dingTailoringTaskArithmetic2025}, which aggregates samples
from two distinct sources: synthetically generated text and posts
from a white supremacist forum.
We treat toxicity detection as the prediction task and use the data
source as the provenance

\paragraph{\textbf{\textcolor{gray}{Civilcomments}} \citep{borkan2019nuanced}}
\textbf{\textcolor{gray}{Civilcomments}} is a comment collection from online articles.
identities that are mentioned in the comment
We follow the same preprocessing procedure in \textsc{WILDS}
\citep{kohWILDSBenchmarkIntheWild2021}.
We use the mention of demographics as a proxy to stereotyping on
certain attribute
which is a common cause of the toxicity and the comment.
We select the mention of black as the provenance.

\paragraph{\textbf{\textcolor{gray}{MultiNLI}} \citep{williams2018broad}}
\textbf{\textcolor{gray}{MultiNLI}} is a Natural Language Inference
(NLI) dataset designed to predict the logical relationship
(entailment, neutral, or contradiction) between a premise and a hypothesis.
Using the training subset, we binarize the task to predict
non-entailment (grouping neutral and contradiction) and define the
provenance based on whether the genre is "fiction."

\paragraph{\textbf{\textcolor{gray}{MIMIC-SubpopBench}} \citep{yangChangeHardCloser2023}}
\textbf{\textcolor{gray}{MIMIC-SubpopBench}} contains clinical notes
recorded during the first 48 hours of a hospital stay, sourced from
the MIMIC-III database.
We follow the preprocessing pipeline from \textsc{SubpopBench}
We define the prediction target as in-hospital mortality and use
patient sex as the provenance.

\section{Experiments}
\label{apd:exp}
\subsection{Datasets}
We subsampled the datasets to introduce spurious correlation in the
training distribution.
We use the sampling parameters $\log\alpha^{tr} = \log\alpha^{val} =
-0.6$, $-1 \leq \log\alpha^{te}
\leq 1$, $|\mathcal{D}^{tr}|:|\mathcal{D}^{val}|:|\mathcal{D}^{te}|=6:2:2$.
With these configurations, we derive 2,162 samples from
\textbf{SHAC},
1,697 samples from \textbf{MIMIC-Location},
6,996 samples from \textbf{HateSpeech},
3,681 samples from \textbf{\textcolor{gray}{MIMIC-SubpopBench}},
8,207 samples in \textbf{\textcolor{gray}{Civilcomments}},
and 62,120 samples from \textbf{\textcolor{gray}{MultiNLI}}.
We alter the random seed in the subsampling process to make the
subsets are seed-dependent and
algorithm-independent (i.e., algorithms are trained and compared
using the same subset).

\subsection{Hyperparameters}
We list the joint distribution of hyperparameters per algorithm for
random search in
\tableref{apd:tab:hyperparameters}.

\begin{table*}[ht]
  \floatconts {apd:tab:hyperparameters}
  {\caption{Hyperparameters, their default values and distributions
  for random search.}}
  {
    \begin{center}
      {
        \begin{tabular}{llll}
          \toprule
          \textbf{Condition} & \textbf{Parameter} & \textbf{Default
          value} & \textbf{Random distribution}\\
          \midrule
          all & learning rate & 1e-5 & $10^{\text{Uniform}(-5, -3.5)}$\\
          & weight decay & 0 & $10^{\text{Uniform}(-6, -2)}$\\
          \midrule
          DANN, CDANN & lambda & 1.0 & $10^{\text{Uniform}(-2, 2)}$\\
          & discriminator weight decay & 0 & $10^{\text{Uniform}(-6, -2)}$\\
          & discriminator steps & 1 & $2^{\text{Uniform}(0, 3)}$\\
          & discriminator width & 256 & $2^{\text{Uniform}(6, 10)}$\\
          & discriminator depth & 3 & $\text{RandomChoice}([3, 4, 5])$\\
          & discriminator dropout & 0 & $\text{RandomChoice}([0, 0.1, 0.5])$\\
          & gradient penalty & 0 & $10^{\text{Uniform}(-2, 1)}$\\
          & adam $\beta_{1}$ & 0.5 & $\text{RandomChoice}([0, 0.5])$\\
          \midrule
          IRM & lambda & 100 & $10^{\text{Uniform}(-1, 5)}$\\
          & iterations of penalty annealing & 500 &
          $10^{\text{Uniform}(0, 4)}$\\
          \midrule
          Mixup & alpha & 0.2 & $10^{\text{Uniform}(0, 4)}$\\
          \midrule
          LISA & alpha & 2.0 & $10^{\text{Uniform}(-1, 1)}$ \\
          & intra-domain mixup ratio & 0.5 & $\text{Uniform}(0, 1)$ \\
          & mixup method & mixup &
          $\text{RandomChoice}([\text{mixup}, \text{cut mixup}])$ \\
          \midrule
          GroupDRO & eta & 0.01 & $10^{\text{Uniform}(-1, 1)}$\\
          \midrule
          {MMD, CORAL} & gamma & 1 & $10^{\text{Uniform}(-1, 1)}$\\
          \midrule
          Fish & meta learning rate & 0.5 &
          $\text{RandomChoice}([0.05, 0.1, 0.5])$\\ \midrule
          MTL & exponential moving average & 0.99 &
          $\text{RandomChoice}([0.5, 0.9, 0.99, 1])$\\
          \midrule
          CAD & lambda & 0.1 & $10^{\text{RandomChoice}([-4, -2, -1,
          0, 1, 2])}$ \\
          & temperature & 0.1 & $\text{RandomChoice}([0.05, 0.1])$\\
          \midrule
          JTT & first stage fraction & 0.5 & $10^{\text{Uniform}(0.2, 0.8)}$ \\
          & lambda $l_{2}$ penalty & 0.1 & $10^{\text{Uniform}(0, 2.5)}$ \\
          \midrule
          DFR & first stage fraction & 0.5 & $10^{\text{Uniform}(0.2, 0.8)}$ \\
          & second stage $l_{2}$ penalty & 0.1 &
          $10^{\text{Uniform}(-2, 0.5)}$ \\
          \midrule
          LfF & amplification degree & 0.1 & $\text{Uniform}(0.05, 0.3)$\\
          \midrule
          DualFilter & mask type & A & $\text{RandomChoice}([D, I, A])$\\
          & mask threshold & 0.5 & $\text{Uniform}(0.5, 0.9)$\\
          & ablation rate & 0.5 & $\text{Uniform}(0.5, 0.9)$\\
          & warm up steps & 50 & $\text{RandomChoice}(10, 25, 50)$\\
          & embedding mask & True & $\text{RandomChoice([False, True])}$\\
          & classifier mask & False & $\text{RandomChoice([False, True])}$\\
          \bottomrule
        \end{tabular}
      }
    \end{center}
  }
\end{table*}

\subsection{Algorithm}
We conduct our evaluation using all algorithms available in the
toolkit, with the notable exception of Backdoor Adjustment
\citep{landeiro2016robust,landeiro2018robust}, an algorithm
deliberately developed to address confounding shift that has
demonstrated utility in addressing confounding by provenance
\citep{dingBackdoorAdjustmentConfounding2024}.
We do not include this algorithm in the current evaluation because it
is not intended for use in the context of end-to-end fine-tuned deep
learning models for NLP, and on account of this underperforms
relative to the algorithms under consideration here.
OOD WGA of BackDoor on each dataset was 0.38, 0.53, 0.39, 0.33, 0.52,
as compared to ERM (0.51 $\pm$ 0.05, 0.51 $\pm$ 0.05, 0.28 $\pm$
  0.02, 0.36 $\pm$ 0.05, 0.83
$\pm$ 0.03) on the \textbf{\textcolor{gray}{Civilcomments}},
\textbf{HateSpeech},
\textbf{\textcolor{gray}{MIMIC-SubpopBench}},
\textbf{\textcolor{gray}{MultiNLI}}, and
\textbf{SHAC} datasets respectively, using the
default hyperparameter.

\subsection{Training Details}
We fine-tune a BERT model (\texttt{bert-base-uncased}) for all
experimental settings \citep{devlin2019bert}.
Following the protocol in \citet{yangChangeHardCloser2023}, we employ early
stopping based on the validation WGA with patience of
three checkpoints.
This strategy is applied to both stages of two-stage algorithms, including
JTT, DFR, and DualFilter.

We train 3,000 steps on \textbf{\textcolor{gray}{MultiNLI}},
1,000 steps on \textbf{MIMIC-Location} and
\textbf{\textcolor{gray}{MIMIC-SubpopBench}},
and 500 steps on \textbf{SHAC},
\textbf{HateSpeech},
and \textbf{\textcolor{gray}{Civilcomments}}.
We checkpoint 10 times for model selection across all experiment settings.
We double the steps for two-stage training algorithms, including JTT,
DFR, and DualFilter.
For the \textbf{\textcolor{gray}{MultiNLI}} dataset, we concatenate
the embedding of premise, hypothesis, their difference,
and their product, aligning with \citet{williams2018broad}.
All training jobs were distributed across two nodes, with the total
experimentation
consuming approximately 1,100 GPU hours.

\subsection{Evaluation Metrics}
\decondtntoolkit supports the calculation of accuracy, f1 score, and
Area Under the Precision-Recall Curve (AURPC), and Expected
Calibration Error (ECE).
These metrics are calculated in provenance-specific, micro-averaged,
macro-averaged, and worst levels.

\section{Infrastructure}
\label{apd:infra}
\subsection{Ablation study on provenance-balanced minibatches}
We intended to follow \textsc{DomainBed} and \textsc{SubpopBench} to
report the ERM with balanced minibatches, which stabilizes the
optimization without altering the learning objective.
An ablation study was conduct on the provenance-balanced-minibatch
technique to investigate its effects on downstream performance.
The ERM without balanced minibatches results in OOD WGA of 0.58,
0.42, 0.30, 0.33, 0.85 using default hyperparameters on the
\textbf{\textcolor{gray}{Civilcomments}},
\textbf{HateSpeech},
\textbf{\textcolor{gray}{MIMIC-SubpopBench}},
\textbf{\textcolor{gray}{MultiNLI}}, and
\textbf{SHAC} datasets, respectively, falling
within ranges of the reported ERM.

\subsection{Training Speed}
Using experiment settings of NVIDIA A100, BERT
(\texttt{bert-base-uncased}), a batch size of 32 per provenance, and
a maximum sequence length of 256, we derived the averaged step time
across available datasets throughout the training process
(\tableref{app:tab:step_time}).

\begin{center}
  \captionof{table}{Mean and standard deviation of step time per algorithm.}
  \label{app:tab:step_time}
  \begin{tabular}{lc}
    \hline
    \textbf{Algorithm} & \textbf{Step Time (s)} \\
    \hline
    CAD & 0.62 (0.19) \\
    CDANN & 0.59 (0.20) \\
    CORAL & 0.60 (0.21) \\
    DANN & 0.59 (0.21) \\
    DFR & 0.22 (0.08) \\
    DownSampling & 0.57 (0.22) \\
    DualFilter & 0.62 (0.17) \\
    ERM & 0.58 (0.21) \\
    Fish & 1.34 (0.46) \\
    GroupDRO & 0.61 (0.21) \\
    IRM & 0.61 (0.21) \\
    JTT & 0.78 (0.26) \\
    LISA & 0.59 (0.22) \\
    LfF & 1.27 (0.44) \\
    MMD & 0.62 (0.21) \\
    MTL & 0.62 (0.22) \\
    Mixup & 0.62 (0.22) \\
    UpSampling & 0.58 (0.22) \\
    \hline
  \end{tabular}
\end{center}

\section{Full Results}
\label{apd:results}
\subsection{Optimal Hyperparameter}
We list the optimal hyperparameters per algorithm for random search in
\tableref{apd:tab:optimalhyper1} and \tableref{apd:tab:optimalhyper2} for
\textbf{Source Datasets}
and \tableref{apd:tab:optimalhyper1_app} and
\tableref{apd:tab:optimalhyper2_app} for
\textbf{\textcolor{gray}{Attribute Datasets}}.

\subsection{Figures}
In-distribution (solid) and OOD (dashed) WGA for ERM on
\textbf{\textcolor{gray}{Attribute Datasets}}
are demonstrated in \figureref{fig:wga_vs_steps_app}.
The relationship between OOD WGA vs. ID WGA and OOD WGA and $alpha$ are on
\textbf{\textcolor{gray}{Attribute Datasets}} illustrated in
\figureref{fig:wga_indicator_app}.

\begin{figure*}[t]
  \small
  \floatconts {fig:wga_vs_steps_app}
  {\caption{In-distribution (solid) and OOD (dashed) WGA for ERM in
      \textbf{\textcolor{gray}{Attribute Datasets}}. Y-axis: Worst Group
      Accuracy (WGA). X-axis: normalized progress of training. The gap
      between the dashed and solid lines shows that models make
      inaccurate predictions in the OOD setting throughout their training
  process. }}
  {\includegraphics[width=\textwidth / 2]{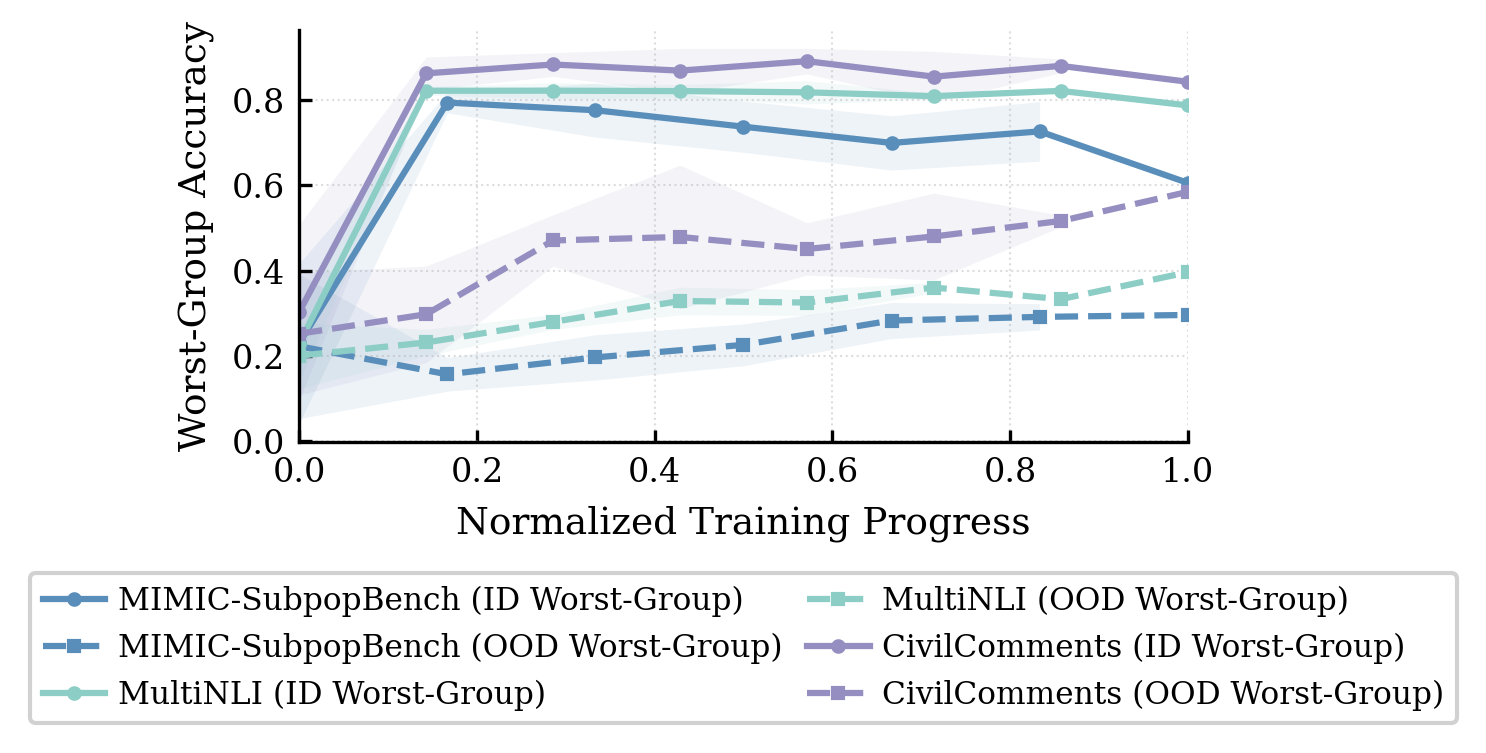}}
\end{figure*}

\begin{figure*}[t]
  \floatconts
  {fig:wga_indicator_app}
  {\caption{WGA is not ``on the line" with respect to ID performance,
      but exhibits a strong linear relationship with the shift parameter
  $\alpha$ in \textbf{\textcolor{gray}{Attribute Datasets}}.}}
  {\includegraphics[width=\textwidth]{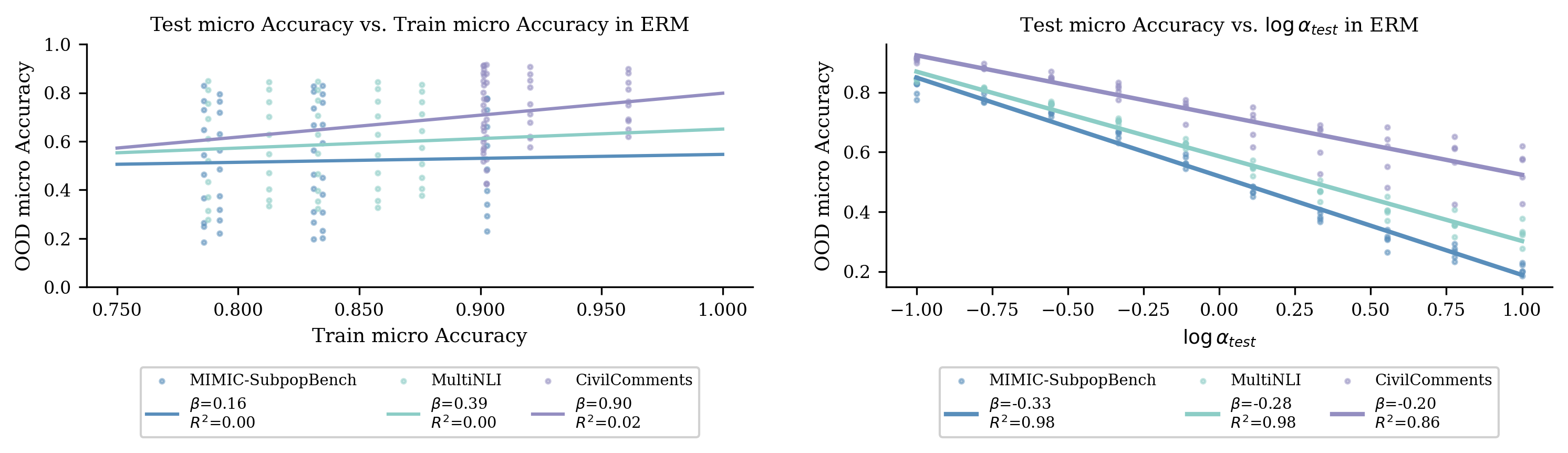}}
\end{figure*}

\subsection{Full Metrics}
We list OOD WGA across algorithms in
\textbf{\textcolor{gray}{Attribute Datasets}}
in \tableref{apd:tab:wga_app}.
Similar to WGA, we consider provenance in metric calculation.
We list the micro-averaged and provenance-macro-averaged
AUPRC per algorithm calculated at $\alpha^{te}=0.6$ for
\textbf{Source Datasets} in
\tableref{apd:tab:microauprc} and \tableref{apd:tab:macroauprc}.
Similarly, the micro-averaged and provenance-macro-averaged
AUPRC for \textbf{\textcolor{gray}{Attribute Datasets}} are listed
in \tableref{apd:tab:microauprc_app} and \tableref{apd:tab:macroauprc_app}.

\begin{table*}[ht]
  \floatconts {apd:tab:optimalhyper1}
  {\caption{Hyperparameters and Optimal Values in
  \textbf{Source Datasets}}}
  {
    \footnotesize
    \begin{tabular}{llccc}
      \toprule
      \textbf{Algorithm} & \textbf{Parameter} &
      \multicolumn{3}{c}{\textbf{Optimal Value}} \\
      \cmidrule(lr){3-5}
      & & \textbf{HateSpeech} &
      \textbf{MIMIC-Location} &
      \textbf{SHAC} \\
      \midrule
      \textbf{CAD} & lambda & 1.0e-03 & 1.0e-04 & 1.0e-03 \\  &
      learning rate & 1.6e-05 & 2.1e-05 & 2.2e-05 \\  & temperature &
      0.1 & 0.1 & 0.1 \\  & weight decay & 3.5e-05 & 1.1e-03 &
      1.0e-05 \\ \midrule \textbf{CDANN} & adam $\beta_{1}$ & 0.5 &
      0.5 & 0 \\  & discriminator depth & 3 & 3 & 3 \\  &
      discriminator dropout & 0.5 & 0 & 0.1 \\  & discriminator
      learning rate & 1.1e-05 & 5.0e-05 & 1.8e-04 \\  & discriminator
      steps & 2 & 1 & 1 \\  & discriminator weight decay & 1.5e-05 &
      0 & 7.2e-04 \\  & discriminator width & 220 & 256 & 829 \\  &
      generator learning rate & 4.6e-05 & 5.0e-05 & 7.1e-05 \\  &
      generator weight decay & 1.4e-06 & 0 & 2.2e-03 \\  & gradient
      penalty & 0.5 & 0 & 0.0 \\  & lambda & 0.8 & 1.0 & 5.8 \\  & lr
      & 1.6e-05 & 1.0e-05 & 2.2e-05 \\  & weight decay & 3.5e-05 & 0
      & 1.0e-05 \\ \midrule \textbf{CORAL} & gamma & 0.3 & 1.6 & 0.3
      \\  & learning rate & 7.7e-05 & 1.4e-05 & 6.7e-05 \\  & weight
      decay & 1.8e-04 & 7.9e-06 & 2.8e-06 \\ \midrule \textbf{DANN} &
      adam $\beta_{1}$ & 0.5 & 0.5 & 0.5 \\  & discriminator depth &
      3 & 3 & 3 \\  & discriminator dropout & 0 & 0 & 0 \\  &
      discriminator learning rate & 5.0e-05 & 5.0e-05 & 5.0e-05 \\  &
      discriminator steps & 1 & 1 & 1 \\  & discriminator weight
      decay & 0 & 0 & 0 \\  & discriminator width & 256 & 256 & 256
      \\  & generator learning rate & 5.0e-05 & 5.0e-05 & 5.0e-05 \\
      & generator weight decay & 0 & 0 & 0 \\  & gradient penalty & 0
      & 0 & 0 \\  & lambda & 1.0 & 1.0 & 1.0 \\  & lr & 1.0e-05 &
      1.0e-05 & 1.0e-05 \\  & weight decay & 0 & 0 & 0 \\ \midrule
      \textbf{DFR} & first stage fraction & 0.7 & 0.6 & 0.3 \\  &
      learning rate & 6.1e-05 & 7.7e-05 & 2.1e-05 \\  & second stage
      $l_{2}$ penalty & 0.3 & 1.6 & 0.0 \\  & weight decay & 1.4e-05
      & 1.8e-04 & 1.1e-03 \\ \midrule \textbf{DownSampling} &
      learning rate & 5.2e-05 & 4.4e-05 & 5.2e-05 \\  & weight decay
      & 2.8e-05 & 7.3e-03 & 2.8e-05 \\ \midrule \textbf{DualFilter} &
      ablation rate & 0.8 & 0.5 & 0.7 \\  & classifier mask & 0 &
      False & 0 \\  & embedding mask & 0 & True & 0 \\  & first stage
      fraction & 0.5 & 0.5 & 0.5 \\  & learning rate & 6.1e-05 &
      1.0e-05 & 5.2e-05 \\  & mask threshold & 0.9 & 0.5 & 0.6 \\  &
      mask type & I & A & A \\  & warm up steps & 50 & 50 & 50 \\  &
      weight decay & 1.4e-05 & 0 & 2.8e-05 \\
      \bottomrule
    \end{tabular}
  }
\end{table*}

\begin{table*}[ht]
  \floatconts {apd:tab:optimalhyper2}
  {\caption{Hyperparameters and Optimal Values in
  \textbf{Source Datasets} (Continued)}}
  {
    \footnotesize
    \begin{tabular}{llccc}
      \toprule
      \textbf{Algorithm} & \textbf{Parameter} &
      \multicolumn{3}{c}{\textbf{Optimal Value}} \\
      \cmidrule(lr){3-5}
      & & \textbf{HateSpeech} &
      \textbf{MIMIC-Location} &
      \textbf{SHAC} \\
      \midrule
      \textbf{ERM} & learning rate & 5.2e-05 & 2.2e-05 & 8.9e-06 \\
      & weight decay & 2.8e-05 & 1.0e-05 & 2.2e-05 \\ \midrule
      \textbf{Fish} & learning rate & 6.7e-05 & 6.7e-05 & 6.7e-05 \\
      & meta learning rate & 0.5 & 0.5 & 0.5 \\  & weight decay &
      2.8e-06 & 2.8e-06 & 2.8e-06 \\ \midrule \textbf{GroupDRO} & eta
      & 1.0e-02 & 0.0 & 0.0 \\  & learning rate & 1.0e-05 & 7.7e-05 &
      8.5e-06 \\  & weight decay & 0 & 1.8e-04 & 4.5e-04 \\ \midrule
      \textbf{IRM} & iterations of penalty annealing & 247 & 247 &
      3211 \\  & lambda & 1.9 & 1.9 & 2.7e+02 \\  & learning rate &
      6.1e-05 & 6.1e-05 & 2.1e-05 \\  & weight decay & 1.4e-05 &
      1.4e-05 & 1.1e-03 \\ \midrule \textbf{JTT} & first stage
      fraction & 0.5 & 0.6 & 0.6 \\  & lambda $l_{2}$ penalty & 10.0
      & 9.6 & 34.9 \\  & learning rate & 1.8e-05 & 7.7e-05 & 6.7e-05
      \\  & weight decay & 5.1e-03 & 1.8e-04 & 2.8e-06 \\ \midrule
      \textbf{LISA} & alpha & 6.2 & 1.3 & 3.2 \\  & intra-domain
      mixup ratio & 0.9 & 0.0 & 1.0 \\  & learning rate & 6.7e-05 &
      8.5e-06 & 2.2e-05 \\  & mixup method & mixup & cutmix & cutmix
      \\  & weight decay & 2.8e-06 & 4.5e-04 & 1.0e-05 \\ \midrule
      \textbf{LfF} & amplification degree & 0.1 & 0.1 & 0.1 \\  &
      learning rate & 1.0e-05 & 1.0e-05 & 8.9e-06 \\  & weight decay
      & 0 & 0 & 2.2e-05 \\ \midrule \textbf{MMD} & gamma & 0.1 & 0.1
      & 0.4 \\  & learning rate & 4.4e-05 & 4.4e-05 & 1.4e-05 \\  &
      weight decay & 7.3e-03 & 7.3e-03 & 7.9e-06 \\ \midrule
      \textbf{MTL} & exponential moving average & 1.0 & 1.0 & 1.0 \\
      & learning rate & 7.7e-05 & 2.2e-05 & 7.7e-05 \\  & weight
      decay & 1.8e-04 & 1.0e-05 & 1.8e-04 \\ \midrule \textbf{Mixup}
      & alpha & 0.6 & 0.3 & 3.1 \\  & learning rate & 2.2e-05 &
      1.4e-05 & 6.1e-05 \\  & weight decay & 1.0e-05 & 7.9e-06 &
      1.4e-05 \\ \midrule \textbf{UpSampling} & learning rate &
      6.7e-05 & 6.1e-05 & 4.4e-05 \\  & weight decay & 2.8e-06 &
      1.4e-05 & 7.3e-03 \\
      \bottomrule
    \end{tabular}
  }
\end{table*}

\begin{table*}[ht]
  \floatconts {apd:tab:optimalhyper1_app}
  {\caption{Hyperparameters and Optimal Values in
  \textbf{\textcolor{gray}{Attribute Datasets}}}}
  {
    \footnotesize
    \begin{tabular}{llccc}
      \toprule
      \textbf{Algorithm} & \textbf{Parameter} &
      \multicolumn{3}{c}{\textbf{Optimal Value}} \\
      \cmidrule(lr){3-5}
      & & \textbf{\textcolor{gray}{Civilcomments}} &
      \textbf{\textcolor{gray}{MIMIC-SubpopBench}} &
      \textbf{\textcolor{gray}{MultiNLI}} \\
      \midrule
      \textbf{CAD} & lambda & 1.0e-03 & 1.0e-04 & 1.0e-03 \\  &
      learning rate & 1.6e-05 & 6.1e-05 & 4.2e-05 \\  & temperature &
      0.1 & 0.1 & 0.1 \\  & weight decay & 3.5e-05 & 1.4e-05 &
      1.7e-05 \\ \midrule \textbf{CDANN} & adam $\beta_{1}$ & 0.5 &
      0.5 & 0.5 \\  & discriminator depth & 3 & 5 & 4 \\  &
      discriminator dropout & 0 & 0.5 & 0.5 \\  & discriminator
      learning rate & 5.0e-05 & 1.2e-04 & 1.3e-05 \\  & discriminator
      steps & 1 & 1 & 2 \\  & discriminator weight decay & 0 &
      5.2e-03 & 1.0e-04 \\  & discriminator width & 256 & 127 & 106
      \\  & generator learning rate & 5.0e-05 & 6.8e-05 & 2.2e-04 \\
      & generator weight decay & 0 & 1.3e-06 & 1.8e-06 \\  & gradient
      penalty & 0 & 0.1 & 0.2 \\  & lambda & 1.0 & 0.0 & 0.0 \\  & lr
      & 1.0e-05 & 9.1e-05 & 7.7e-05 \\  & weight decay & 0 & 6.2e-04
      & 1.8e-04 \\ \midrule \textbf{CORAL} & gamma & 0.6 & 0.5 & 0.6
      \\  & learning rate & 5.2e-05 & 2.1e-05 & 4.2e-05 \\  & weight
      decay & 2.8e-05 & 1.1e-03 & 1.7e-05 \\ \midrule \textbf{DANN} &
      adam $\beta_{1}$ & 0.5 & 0.5 & 0.5 \\  & discriminator depth &
      3 & 5 & 4 \\  & discriminator dropout & 0.5 & 0.1 & 0.5 \\  &
      discriminator learning rate & 1.1e-05 & 1.7e-05 & 1.3e-05 \\  &
      discriminator steps & 2 & 4 & 2 \\  & discriminator weight
      decay & 1.5e-05 & 4.2e-03 & 1.0e-04 \\  & discriminator width &
      220 & 497 & 106 \\  & generator learning rate & 4.6e-05 &
      6.0e-05 & 2.2e-04 \\  & generator weight decay & 1.4e-06 &
      1.9e-03 & 1.8e-06 \\  & gradient penalty & 0.5 & 0.0 & 0.2 \\
      & lambda & 0.8 & 4.9 & 0.0 \\  & lr & 1.6e-05 & 1.8e-05 &
      7.7e-05 \\  & weight decay & 3.5e-05 & 5.1e-03 & 1.8e-04
      \\ \midrule \textbf{DFR} & first stage fraction & 0.7 & 0.3 &
      0.6 \\  & learning rate & 6.1e-05 & 2.1e-05 & 6.7e-05 \\  &
      second stage $l_{2}$ penalty & 0.3 & 0.0 & 0.0 \\  & weight
      decay & 1.4e-05 & 1.1e-03 & 2.8e-06 \\ \midrule
      \textbf{DownSampling} & learning rate & 4.2e-05 & 8.9e-06 &
      9.9e-06 \\  & weight decay & 1.7e-05 & 2.2e-05 & 6.3e-06
      \\ \midrule \textbf{DualFilter} & ablation rate & 0.7 & 0.8 &
      0.8 \\  & classifier mask & 0 & 0 & 0 \\  & embedding mask & 0
      & 0 & 0 \\  & first stage fraction & 0.5 & 0.5 & 0.5 \\  &
      learning rate & 5.2e-05 & 6.1e-05 & 6.1e-05 \\  & mask
      threshold & 0.6 & 0.9 & 0.9 \\  & mask type & A & I & I \\  &
      warm up steps & 50 & 50 & 50 \\  & weight decay & 2.8e-05 &
      1.4e-05 & 1.4e-05 \\
      \bottomrule
    \end{tabular}
  }
\end{table*}

\begin{table*}[ht]
  \floatconts {apd:tab:optimalhyper2_app}
  {\caption{Hyperparameters and Optimal Values in
  \textbf{\textcolor{gray}{Attribute Datasets}} (Continued)}}
  {
    \footnotesize
    \begin{tabular}{llccc}
      \toprule
      \textbf{Algorithm} & \textbf{Parameter} &
      \multicolumn{3}{c}{\textbf{Optimal Value}} \\
      \cmidrule(lr){3-5}
      & & \textbf{\textcolor{gray}{Civilcomments}} &
      \textbf{\textcolor{gray}{MIMIC-SubpopBench}} &
      \textbf{\textcolor{gray}{MultiNLI}} \\
      \midrule
      \textbf{ERM} & learning rate & 4.4e-05 & 8.9e-06 & 4.2e-05 \\
      & weight decay & 7.3e-03 & 2.2e-05 & 1.7e-05 \\ \midrule
      \textbf{Fish} & learning rate & 9.1e-05 & 8.9e-06 & 1.6e-05 \\
      & meta learning rate & 0.1 & 0.1 & 0.5 \\  & weight decay &
      6.2e-04 & 2.2e-05 & 3.5e-05 \\ \midrule \textbf{GroupDRO} & eta
      & 0.0 & 0.0 & 0.0 \\  & learning rate & 9.9e-06 & 1.8e-05 &
      6.7e-05 \\  & weight decay & 6.3e-06 & 5.1e-03 & 2.8e-06
      \\ \midrule \textbf{IRM} & iterations of penalty annealing &
      3001 & 3211 & 3775 \\  & lambda & 29.4 & 2.7e+02 & 7.5e+04 \\
      & learning rate & 2.2e-05 & 2.1e-05 & 7.7e-05 \\  & weight
      decay & 1.0e-05 & 1.1e-03 & 1.8e-04 \\ \midrule \textbf{JTT} &
      first stage fraction & 0.3 & 0.7 & 0.8 \\  & lambda $l_{2}$
      penalty & 2.9 & 1.3e+02 & 34.7 \\  & learning rate & 2.1e-05 &
      6.1e-05 & 9.1e-05 \\  & weight decay & 1.1e-03 & 1.4e-05 &
      6.2e-04 \\ \midrule \textbf{LISA} & alpha & 0.2 & 0.1 & 0.2 \\
      & intra-domain mixup ratio & 0.1 & 0.9 & 0.6 \\  & learning
      rate & 4.2e-05 & 8.9e-06 & 5.2e-05 \\  & mixup method & mixup &
      mixup & mixup \\  & weight decay & 1.7e-05 & 2.2e-05 & 2.8e-05
      \\ \midrule \textbf{LfF} & amplification degree & 0.1 & 0.1 &
      0.1 \\  & learning rate & 8.5e-06 & 8.5e-06 & 1.4e-05 \\  &
      weight decay & 4.5e-04 & 4.5e-04 & 7.9e-06 \\ \midrule
      \textbf{MMD} & gamma & 0.1 & 0.1 & 0.1 \\  & learning rate &
      4.4e-05 & 4.4e-05 & 4.4e-05 \\  & weight decay & 7.3e-03 &
      7.3e-03 & 7.3e-03 \\ \midrule \textbf{MTL} & exponential moving
      average & 0.9 & 0.9 & 1.0 \\  & learning rate & 1.4e-05 &
      6.7e-05 & 5.2e-05 \\  & weight decay & 7.9e-06 & 2.8e-06 &
      2.8e-05 \\ \midrule \textbf{Mixup} & alpha & 0.1 & 0.6 & 0.2
      \\  & learning rate & 8.9e-06 & 2.2e-05 & 7.7e-05 \\  & weight
      decay & 2.2e-05 & 1.0e-05 & 1.8e-04 \\ \midrule
      \textbf{UpSampling} & learning rate & 5.2e-05 & 1.0e-05 &
      6.7e-05 \\  & weight decay & 2.8e-05 & 0 & 2.8e-06 \\
      \bottomrule
    \end{tabular}
  }
\end{table*}

\begin{table*}[h]
  \floatconts {apd:tab:wga_app}
  {\caption{OOD WGA across algorithms in
  \textbf{\textcolor{gray}{Attribute Datasets}}}}
  {
    \begin{center}
      \begin{tabular}{lcccc}
        \toprule
        \textbf{Algorithm} &
        \textbf{\textcolor{gray}{MIMIC-SubpopBench}} &
        \textbf{\textcolor{gray}{MultiNLI}} &
        \textbf{\textcolor{gray}{Civilcomments}} & \textbf{Avg} \\
        \midrule
        ERM & 0.28 $\pm$ 0.02 & 0.36 $\pm$ 0.05 & 0.51 $\pm$ 0.05 & 0.38 \\
        UpSampling & 0.46 $\pm$ 0.03 & 0.56 $\pm$ 0.02 & 0.67 $\pm$
        0.03 & 0.56 \\
        DownSampling & 0.51 $\pm$ 0.05 & 0.63 $\pm$ 0.00 & 0.72 $\pm$
        0.02 & 0.62 \\
        \midrule
        CAD & 0.30 $\pm$ 0.06 & 0.34 $\pm$ 0.02 & 0.58 $\pm$ 0.06 & 0.41 \\
        CDANN & 0.31 $\pm$ 0.02 & 0.33 $\pm$ 0.04 & 0.50 $\pm$ 0.07 & 0.38 \\
        CORAL & 0.30 $\pm$ 0.03 & 0.32 $\pm$ 0.04 & 0.53 $\pm$ 0.09 & 0.39 \\
        DANN & 0.27 $\pm$ 0.02 & 0.31 $\pm$ 0.03 & 0.45 $\pm$ 0.04 & 0.34 \\
        DFR & 0.37 $\pm$ 0.05 & 0.60 $\pm$ 0.01 & 0.58 $\pm$ 0.07 & 0.52 \\
        DualFilter & 0.31 $\pm$ 0.06 & 0.40 $\pm$ 0.02 & 0.53 $\pm$
        0.05 & 0.41 \\
        Fish & 0.30 $\pm$ 0.03 & 0.31 $\pm$ 0.01 & 0.49 $\pm$ 0.05 & 0.37 \\
        GroupDRO & 0.28 $\pm$ 0.03 & 0.35 $\pm$ 0.04 & 0.52 $\pm$ 0.05 & 0.38 \\
        IRM & 0.31 $\pm$ 0.04 & 0.40 $\pm$ 0.04 & 0.54 $\pm$ 0.07 & 0.42 \\
        JTT & 0.28 $\pm$ 0.06 & 0.34 $\pm$ 0.07 & 0.52 $\pm$ 0.06 & 0.38 \\
        LISA & 0.47 $\pm$ 0.02 & 0.64 $\pm$ 0.01 & 0.67 $\pm$ 0.05 & 0.60 \\
        LfF & 0.32 $\pm$ 0.06 & 0.35 $\pm$ 0.04 & 0.53 $\pm$ 0.03 & 0.40 \\
        MMD & 0.34 $\pm$ 0.03 & 0.35 $\pm$ 0.03 & 0.52 $\pm$ 0.05 & 0.40 \\
        MTL & 0.34 $\pm$ 0.06 & 0.36 $\pm$ 0.04 & 0.53 $\pm$ 0.04 & 0.41 \\
        Mixup & 0.29 $\pm$ 0.02 & 0.36 $\pm$ 0.05 & 0.53 $\pm$ 0.06 & 0.39 \\
        \bottomrule
      \end{tabular}
    \end{center}
  }
\end{table*}

\begin{table*}[h]
  \floatconts {apd:tab:microauprc}
  {\caption{OOD micro-averaged Area Under the Precision-Recall Curve
  in \textbf{Source Datasets}}}
  {
    \begin{tabular}{lcccc}
      \toprule
      \textbf{Algorithm} & \textbf{SHAC} &
      \textbf{MIMIC-Location} &
      \textbf{HateSpeech} & \textbf{Avg} \\
      \midrule
      ERM & 0.91 $\pm$ 0.02 & 0.48 $\pm$ 0.04 & 0.69 $\pm$ 0.06 & 0.69 \\
      UpSampling & 0.96 $\pm$ 0.01 & 0.53 $\pm$ 0.03 & 0.76 $\pm$ 0.03 & 0.75 \\
      DownSampling & 0.95 $\pm$ 0.01 & 0.54 $\pm$ 0.02 & 0.80 $\pm$
      0.02 & 0.76 \\
      \midrule
      CAD & 0.92 $\pm$ 0.02 & 0.49 $\pm$ 0.03 & 0.66 $\pm$ 0.03 & 0.69 \\
      CDANN & 0.95 $\pm$ 0.03 & 0.49 $\pm$ 0.02 & 0.66 $\pm$ 0.03 & 0.70 \\
      CORAL & 0.95 $\pm$ 0.01 & 0.50 $\pm$ 0.03 & 0.70 $\pm$ 0.04 & 0.72 \\
      DANN & 0.95 $\pm$ 0.02 & 0.50 $\pm$ 0.04 & 0.71 $\pm$ 0.03 & 0.72 \\
      DFR & 0.93 $\pm$ 0.01 & 0.49 $\pm$ 0.02 & 0.69 $\pm$ 0.03 & 0.71 \\
      DualFilter & 0.94 $\pm$ 0.02 & 0.49 $\pm$ 0.02 & 0.72 $\pm$ 0.03 & 0.72 \\
      Fish & 0.93 $\pm$ 0.02 & 0.49 $\pm$ 0.03 & 0.66 $\pm$ 0.03 & 0.69 \\
      GroupDRO & 0.90 $\pm$ 0.03 & 0.50 $\pm$ 0.02 & 0.67 $\pm$ 0.02 & 0.69 \\
      IRM & 0.93 $\pm$ 0.02 & 0.50 $\pm$ 0.03 & 0.69 $\pm$ 0.05 & 0.70 \\
      JTT & 0.92 $\pm$ 0.05 & 0.49 $\pm$ 0.03 & 0.67 $\pm$ 0.03 & 0.70 \\
      LISA & 0.94 $\pm$ 0.04 & 0.52 $\pm$ 0.02 & 0.77 $\pm$ 0.01 & 0.74 \\
      LfF & 0.87 $\pm$ 0.05 & 0.49 $\pm$ 0.02 & 0.61 $\pm$ 0.03 & 0.66 \\
      MMD & 0.93 $\pm$ 0.02 & 0.47 $\pm$ 0.02 & 0.69 $\pm$ 0.02 & 0.70 \\
      MTL & 0.92 $\pm$ 0.02 & 0.48 $\pm$ 0.02 & 0.64 $\pm$ 0.04 & 0.68 \\
      Mixup & 0.91 $\pm$ 0.04 & 0.50 $\pm$ 0.02 & 0.71 $\pm$ 0.02 & 0.71 \\
      \bottomrule
    \end{tabular}
  }
\end{table*}

\begin{table*}[h]
  \floatconts {apd:tab:macroauprc}
  {\caption{OOD macro-averaged Area Under the Precision-Recall Curve
  on \textbf{Source Datasets}}}
  {
    \begin{tabular}{lcccc}
      \toprule
      \textbf{Algorithm} & \textbf{SHAC} &
      \textbf{MIMIC-Location} &
      \textbf{HateSpeech} & \textbf{Avg} \\
      \midrule
      ERM & 0.92 $\pm$ 0.03 & 0.57 $\pm$ 0.02 & 0.73 $\pm$ 0.03 & 0.74 \\
      UpSampling & 0.96 $\pm$ 0.01 & 0.57 $\pm$ 0.03 & 0.75 $\pm$ 0.03 & 0.76 \\
      DownSampling & 0.93 $\pm$ 0.02 & 0.54 $\pm$ 0.02 & 0.71 $\pm$
      0.01 & 0.73 \\
      \midrule
      CAD & 0.94 $\pm$ 0.02 & 0.57 $\pm$ 0.02 & 0.74 $\pm$ 0.02 & 0.75 \\
      CDANN & 0.94 $\pm$ 0.02 & 0.55 $\pm$ 0.02 & 0.72 $\pm$ 0.02 & 0.74 \\
      CORAL & 0.91 $\pm$ 0.02 & 0.57 $\pm$ 0.02 & 0.72 $\pm$ 0.03 & 0.74 \\
      DANN & 0.89 $\pm$ 0.08 & 0.54 $\pm$ 0.02 & 0.72 $\pm$ 0.03 & 0.72 \\
      DFR & 0.94 $\pm$ 0.02 & 0.55 $\pm$ 0.03 & 0.75 $\pm$ 0.03 & 0.75 \\
      DualFilter & 0.95 $\pm$ 0.01 & 0.57 $\pm$ 0.02 & 0.75 $\pm$ 0.04 & 0.76 \\
      Fish & 0.94 $\pm$ 0.02 & 0.57 $\pm$ 0.02 & 0.72 $\pm$ 0.03 & 0.74 \\
      GroupDRO & 0.93 $\pm$ 0.04 & 0.57 $\pm$ 0.02 & 0.73 $\pm$ 0.03 & 0.74 \\
      IRM & 0.93 $\pm$ 0.01 & 0.57 $\pm$ 0.02 & 0.72 $\pm$ 0.01 & 0.74 \\
      JTT & 0.93 $\pm$ 0.06 & 0.57 $\pm$ 0.03 & 0.73 $\pm$ 0.02 & 0.74 \\
      LISA & 0.93 $\pm$ 0.05 & 0.55 $\pm$ 0.02 & 0.71 $\pm$ 0.02 & 0.73 \\
      LfF & 0.91 $\pm$ 0.04 & 0.55 $\pm$ 0.01 & 0.69 $\pm$ 0.02 & 0.72 \\
      MMD & 0.91 $\pm$ 0.03 & 0.55 $\pm$ 0.01 & 0.73 $\pm$ 0.01 & 0.73 \\
      MTL & 0.90 $\pm$ 0.04 & 0.55 $\pm$ 0.03 & 0.67 $\pm$ 0.02 & 0.71 \\
      Mixup & 0.84 $\pm$ 0.05 & 0.57 $\pm$ 0.03 & 0.72 $\pm$ 0.02 & 0.71 \\
      \bottomrule
    \end{tabular}
  }
\end{table*}

\begin{table*}[h]
  \floatconts {apd:tab:microauprc_app}
  {\caption{OOD micro-averaged Area Under the Precision-Recall Curve
  on \textbf{\textcolor{gray}{Attribute Datasets}}}}
  {
    \begin{tabular}{lcccc}
      \toprule
      \textbf{Algorithm} &
      \textbf{\textcolor{gray}{MIMIC-SubpopBench}} &
      \textbf{\textcolor{gray}{MultiNLI}} &
      \textbf{\textcolor{gray}{Civilcomments}} & \textbf{Avg} \\
      \midrule
      ERM & 0.39 $\pm$ 0.02 & 0.49 $\pm$ 0.02 & 0.66 $\pm$ 0.10 & 0.51 \\
      UpSampling & 0.48 $\pm$ 0.03 & 0.64 $\pm$ 0.02 & 0.83 $\pm$ 0.01 & 0.65 \\
      DownSampling & 0.56 $\pm$ 0.03 & 0.72 $\pm$ 0.01 & 0.85 $\pm$
      0.02 & 0.71 \\
      \midrule
      CAD & 0.39 $\pm$ 0.02 & 0.46 $\pm$ 0.02 & 0.71 $\pm$ 0.04 & 0.52 \\
      CDANN & 0.40 $\pm$ 0.02 & 0.46 $\pm$ 0.02 & 0.68 $\pm$ 0.05 & 0.51 \\
      CORAL & 0.40 $\pm$ 0.02 & 0.47 $\pm$ 0.03 & 0.71 $\pm$ 0.02 & 0.53 \\
      DANN & 0.38 $\pm$ 0.01 & 0.48 $\pm$ 0.05 & 0.64 $\pm$ 0.05 & 0.50 \\
      DFR & 0.45 $\pm$ 0.02 & 0.68 $\pm$ 0.01 & 0.71 $\pm$ 0.03 & 0.61 \\
      DualFilter & 0.42 $\pm$ 0.03 & 0.53 $\pm$ 0.03 & 0.71 $\pm$ 0.02 & 0.55 \\
      Fish & 0.38 $\pm$ 0.01 & 0.45 $\pm$ 0.02 & 0.60 $\pm$ 0.06 & 0.48 \\
      GroupDRO & 0.37 $\pm$ 0.01 & 0.49 $\pm$ 0.03 & 0.65 $\pm$ 0.02 & 0.50 \\
      IRM & 0.41 $\pm$ 0.02 & 0.50 $\pm$ 0.04 & 0.67 $\pm$ 0.03 & 0.53 \\
      JTT & 0.42 $\pm$ 0.04 & 0.49 $\pm$ 0.04 & 0.66 $\pm$ 0.04 & 0.52 \\
      LISA & 0.51 $\pm$ 0.04 & 0.72 $\pm$ 0.02 & 0.84 $\pm$ 0.03 & 0.69 \\
      LfF & 0.40 $\pm$ 0.03 & 0.48 $\pm$ 0.03 & 0.64 $\pm$ 0.06 & 0.51 \\
      MMD & 0.40 $\pm$ 0.01 & 0.48 $\pm$ 0.02 & 0.69 $\pm$ 0.03 & 0.52 \\
      MTL & 0.42 $\pm$ 0.05 & 0.49 $\pm$ 0.04 & 0.63 $\pm$ 0.03 & 0.51 \\
      Mixup & 0.39 $\pm$ 0.02 & 0.48 $\pm$ 0.03 & 0.67 $\pm$ 0.01 & 0.51 \\
      \bottomrule
    \end{tabular}
  }
\end{table*}

\begin{table*}[h]
  \floatconts {apd:tab:macroauprc_app}
  {\caption{OOD macro-averaged Area Under the Precision-Recall Curve
  on \textbf{\textcolor{gray}{Attribute Datasets}}}}
  {
    \begin{tabular}{lcccc}
      \toprule
      \textbf{Algorithm} &
      \textbf{\textcolor{gray}{MIMIC-SubpopBench}} &
      \textbf{\textcolor{gray}{MultiNLI}} &
      \textbf{\textcolor{gray}{Civilcomments}} & \textbf{Avg} \\
      \midrule
      ERM & 0.54 $\pm$ 0.03 & 0.67 $\pm$ 0.01 & 0.84 $\pm$ 0.05 & 0.69 \\
      UpSampling & 0.56 $\pm$ 0.02 & 0.68 $\pm$ 0.01 & 0.84 $\pm$ 0.02 & 0.70 \\
      DownSampling & 0.55 $\pm$ 0.01 & 0.69 $\pm$ 0.01 & 0.82 $\pm$
      0.03 & 0.69 \\
      \midrule
      CAD & 0.53 $\pm$ 0.02 & 0.66 $\pm$ 0.01 & 0.84 $\pm$ 0.03 & 0.68 \\
      CDANN & 0.55 $\pm$ 0.02 & 0.66 $\pm$ 0.01 & 0.83 $\pm$ 0.03 & 0.68 \\
      CORAL & 0.54 $\pm$ 0.02 & 0.67 $\pm$ 0.01 & 0.84 $\pm$ 0.03 & 0.68 \\
      DANN & 0.54 $\pm$ 0.01 & 0.67 $\pm$ 0.01 & 0.84 $\pm$ 0.02 & 0.68 \\
      DFR & 0.57 $\pm$ 0.02 & 0.70 $\pm$ 0.01 & 0.84 $\pm$ 0.01 & 0.70 \\
      DualFilter & 0.54 $\pm$ 0.03 & 0.69 $\pm$ 0.01 & 0.86 $\pm$ 0.01 & 0.69 \\
      Fish & 0.53 $\pm$ 0.02 & 0.66 $\pm$ 0.01 & 0.79 $\pm$ 0.05 & 0.66 \\
      GroupDRO & 0.52 $\pm$ 0.02 & 0.68 $\pm$ 0.01 & 0.83 $\pm$ 0.03 & 0.68 \\
      IRM & 0.54 $\pm$ 0.02 & 0.67 $\pm$ 0.01 & 0.83 $\pm$ 0.02 & 0.68 \\
      JTT & 0.54 $\pm$ 0.03 & 0.67 $\pm$ 0.01 & 0.83 $\pm$ 0.03 & 0.68 \\
      LISA & 0.56 $\pm$ 0.01 & 0.71 $\pm$ 0.01 & 0.85 $\pm$ 0.03 & 0.71 \\
      LfF & 0.53 $\pm$ 0.03 & 0.66 $\pm$ 0.01 & 0.79 $\pm$ 0.05 & 0.66 \\
      MMD & 0.55 $\pm$ 0.01 & 0.67 $\pm$ 0.01 & 0.84 $\pm$ 0.02 & 0.69 \\
      MTL & 0.53 $\pm$ 0.02 & 0.64 $\pm$ 0.03 & 0.75 $\pm$ 0.02 & 0.64 \\
      Mixup & 0.53 $\pm$ 0.03 & 0.66 $\pm$ 0.02 & 0.83 $\pm$ 0.02 & 0.67 \\
      \bottomrule
    \end{tabular}
  }
\end{table*}

\end{document}